\pdfoutput=1

\documentclass[pdflatex,sn-mathphys]{sn-jnl}

\usepackage{graphicx}
\usepackage{amsmath, amssymb, amsfonts}
\usepackage{amsthm}
\usepackage{xcolor}
\usepackage{algorithm}
\usepackage{algorithmicx}
\usepackage{algpseudocode}

\usepackage{Sty/mcr}
\usepackage{bm}
\usepackage{subcaption}
\theoremstyle{definition} 
\newtheorem{lemma}{Lemma}
\algrenewcommand{\algorithmicrequire}{\textbf{Input:}}
\algrenewcommand{\algorithmicensure}{\textbf{Output:}}

\theoremstyle{thmstyleone}

\theoremstyle{thmstyletwo}

\newtheorem{remark}{Remark}

\theoremstyle{thmstylethree}

\begin{document}

\title[Post-Transfer Learning Statistical Inference in High-Dimensional Regression]{Post-Transfer Learning Statistical Inference in High-Dimensional Regression}

\author[1,2]{Nguyen Vu Khai Tam} 

\author[1,2]{Cao Huyen My} 

\author*[1,2,3]{Vo Nguyen Le Duy}\email{duyvnl@uit.edu.vn}

\affil[1]{\orgname{University of Information Technology}, \city{Ho Chi Minh City}, \country{Vietnam}}

\affil[2]{\orgname{Vietnam National University}, \city{Ho Chi Minh City}, \country{Vietnam}}

\affil[3]{\orgname{RIKEN AIP}, \city{Tokyo}, \country{Japan}}

\abstract{

Transfer learning (TL) for high-dimensional regression (HDR) is an important problem in machine learning, particularly when dealing with limited sample size in the target task.
However, there currently lacks a method to quantify the statistical significance of the relationship between features and the response in TL-HDR settings.
In this paper, we introduce a novel statistical inference framework for assessing the reliability of feature selection in TL-HDR, called \emph{PTL-SI} (\underline{P}ost-\underline{TL} \underline{S}tatistical \underline{I}nference).
The core contribution of PTL-SI is its ability to provide valid $p$-values to features selected in TL-HDR, thereby rigorously controlling the false positive rate (FPR) at desired significance level $\alpha$ (e.g., 0.05).
Furthermore, we enhance statistical power by incorporating a strategic divide-and-conquer approach into our framework.
We demonstrate the validity and effectiveness of the proposed PTL-SI through extensive experiments on both synthetic and real-world high-dimensional datasets, confirming its theoretical properties and utility in testing the reliability of feature selection in TL scenarios.
}

\keywords{Transfer Learning, High-dimensional Regression, Uncertainty Quantification, Statistical Hypothesis Testing, $p$-value}

\maketitle

\section{Introduction}\label{sec1}
\begin{figure}[htbp]
    \centering
        \includegraphics[width=1\linewidth]{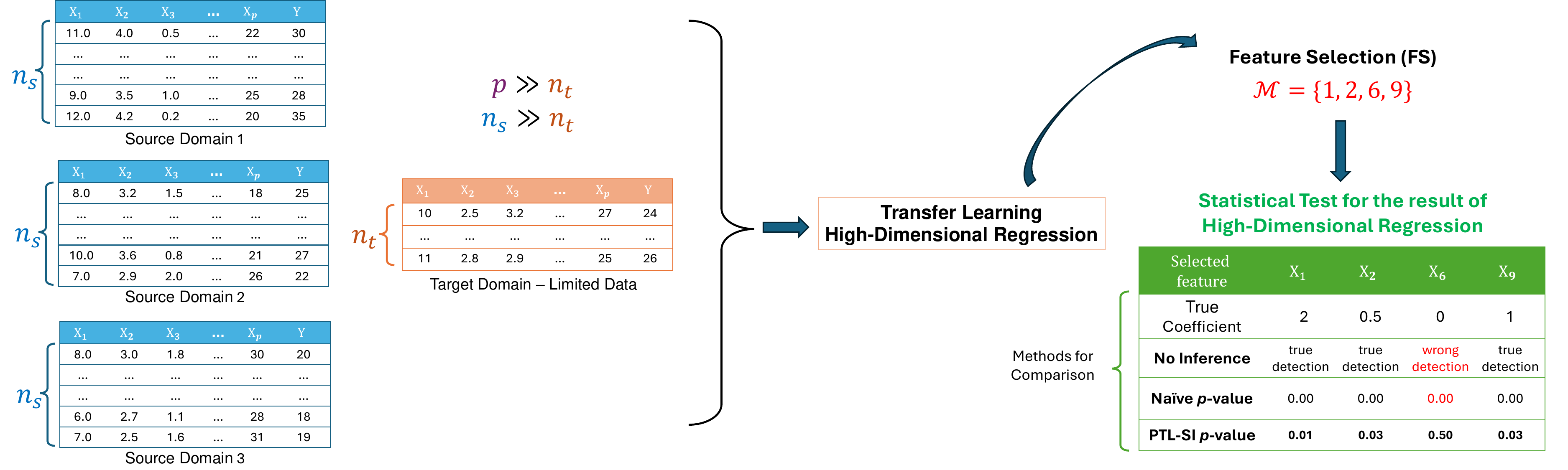}
        \caption{Illustration of the proposed PTL-SI method. Conducting post-transfer learning analysis without statistical inference in high-dimensional regression may lead to the erroneous identification of irrelevant features (e.g., $X_6$). The naive $p$-value is even small for falsely selected feature. In contrast, with the proposed PTL-SI method, we can identify both false positives (FPs) and true positives (TPs). i.e., large $p$-values for irrelevant features and small $p$-values for truly informative ones, thereby enhancing the reliability of feature selection after transfer learning.}
    \label{fig:introduction}
\end{figure}
High-dimensional regression (HDR), where the sample size is much smaller than the number of features, is a fundamental challenge in machine learning and statistics, particularly in data-scarce settings \cite{li2022transfer, he2024transfusion}.  
Such scenarios frequently occur in fields such as genomics \cite{mei2011gene}, financial modeling \cite{yang2016nonnegative}, and medical image analysis \cite{shin2016deep}, where data are inherently high-dimensional but limited in sample size. 
In these contexts, classical learning algorithms often struggle, leading to poorly predictive performance.
\emph{Transfer learning (TL)} has emerged as a promising approach to tackle this challenge.
Rather than relying solely on the limited data from the target task, transfer learning incorporates knowledge from one or more related source tasks. 
By leveraging the potentially abundant and rich information from these auxiliary sources, TL can significantly enhance the efficacy of the regression model for the target task.
For instance, in genomic studies of rare diseases, supplementing the limited target dataset with information from larger, related source studies can reveal critical genetic markers that might remain undetected if the analysis were confined to the target sample alone.

\vspace{8pt}

However, a critical limitation remains: the lack of a principled framework for statistical inference. 
This limitation hinders the interpretability and broader adoption of transfer learning in scientific research, where evaluating the statistical significance of discovered features is often essential. 
In particular, controlling the rate of false discoveries when identifying relevant features is crucial for ensuring the reliability of the findings.
In HDR settings, feature selection (FS) aims to identify the truly influential features from a vast number of irrelevant ones. 
The inclusion of irrelevant features (i.e., false positives) can lead to misleading conclusions and potentially harmful consequences.
For example, incorrectly identifying non-causal genetic markers as risk factors could lead to unnecessary, costly, and potentially detrimental interventions for patients.
Consequently, developing an inference method that can control the false positive rate (FPR) is crucially important.
We note that it is also important to control the false negative rate (FNR).
Following the established statistical practice, we propose a method that provides theoretical control over the false positive rate (FPR) at a pre-specified level (e.g., $\alpha = 0.05$), while simultaneously aiming to minimize the FNR---equivalently, to maximize the true positive rate (TPR)

\vspace{8pt}

Conducting valid statistical inference to control the FPR in TL-HDR poses a significant challenge.
The central difficulty arises from the fact that the features selected for inference are determined by applying TL-HDR algorithms to the data.
This data-dependent selection process violates a key assumption of classical inference methods, which require the set of features to be fixed in advance.
To address this challenge, our approach is inspired by the framework of \emph{Selective Inference (SI)} \cite{lee2016exact}.
However, directly applying existing SI techniques is not feasible, as they are typically designed for specific models and well-defined selection procedures. 
Therefore, we develop a new SI-based method that carefully accommodates the structure of TL-HDR algorithms and their associated feature selection mechanisms.

\vspace{8pt}

In this paper, we primarily focus on developing a statistical inference method for TransFusion \cite{he2024transfusion}, the most recent approach that has demonstrated superior performance on the TL-HDR task.
Furthermore, to demonstrate the adaptability of our proposed method, we show that it can be extended to perform statistical inference for the method introduced in \cite{li2022transfer}, the most cited study in the TL-HDR literature.

\vspace{\baselineskip}

\noindent\textbf{Contributions.}
Our contributions are as follows:
\begin{itemize}
    \item We mathematically formulate the problem of testing FS results in TL-HDR within a hypothesis testing framework. This presents a distinct challenge, as it requires addressing the effects of TL to maintain valid control over the FPR.
    
   \vspace{5pt}
    \item 
    
     We propose a novel statistical method, called \emph{PTL-SI} (\underline{P}ost-\underline{T}ransfer \underline{L}earning \underline{S}tatistical \underline{I}nference), designed to conduct the proposed hypothesis test.
    We demonstrate that achieving control over the FPR in the TL-HDR setting is feasible. To the best of our knowledge, this is the first method capable of conducting valid inference within the context of TL-HDR.
    Furthermore, we present strategic approaches aimed at maximizing the TPR, thereby minimizing the FNR.
    
    \vspace{5pt}
    \item We perform comprehensive experiments on both synthetic and real-world datasets to rigorously validate our theoretical results and showcase the superior performance of the proposed PTL-SI method.
   
\end{itemize}

Figure \ref{fig:introduction} shows an illustrative example of the problem we consider in this paper and the importance of the proposed method. For reproducibility, our implementation is available at: \href{https://github.com/22520896/PTL_SI}{https://github.com/22520896/PTL\_SI}.

\vspace{\baselineskip} 
\noindent\textbf{Related works.}
Traditional statistical inference for FS results often encounters issues with the validity of $p$-values. A common problem arises from the use of naive $p$-values, which are computed under the assumption that the selected features were fixed in advance. However, in the context of TL-HDR, this assumption is violated, making the naive $p$-values invalid. Data splitting (DS) can provide a potential solution by dividing the data, but it comes with the trade-off of reducing the available data for both feature selection and inference, which can diminish statistical power. Moreover, DS is not always feasible, especially when the data exhibits correlations.

\vspace{5pt}

The SI framework has been extensively studied for conducting inference on features selected by FS methods in linear models. Initially introduced for Lasso \cite{lee2016exact}, the core idea of SI is to perform inference conditional on the FS process. This approach helps mitigate bias introduced by the FS step, enabling the computation of valid $p$-values. The seminal work on SI has paved the way for further research on SI in FS \cite{loftus2014significance, fithian2014optimal, tibshirani2016exact, yang2016selective, suzumura2017selective, sugiyama2021more, duy2022more}.
However, these methods are typically designed for scenarios where there is sufficient data for regression on the target task. 
In scenarios where the target task suffers from limited data and TL is necessary to leverage information from related source tasks, these existing methods lose their validity, as they do not account for the effects of the TL process.

\vspace{5pt}

A closely related work and the main motivation for this study is \cite{loi2024statistical}, where the authors propose a framework to compute valid $p$-values for FS results after optimal transport (OT)-based domain adaptation (DA).
This study primarily focuses on developing statistical inference techniques specifically designed for the OT-based DA approach, which is completely different from the TL approaches considered in this paper.
Moreover, the setting in \cite{loi2024statistical} is limited to adaptation from a single source domain to a target domain in low-dimensional problems.
In contrast, we examine a scenario that involves transferring knowledge from \emph{multiple} source tasks to a target task in high-dimensional problems.
Therefore, the method in \cite{loi2024statistical} cannot be applied to our setting.

\section{Problem Statement}\label{sec2}

We consider a transfer learning setting with a single target task and $K$ source tasks.
For the target task, we consider 
\begin{equation}\label{eq:y0}
    \bm{Y}^{(0)} = \left(Y^{(0)}_1, Y^{(0)}_2, \dots, Y^{(0)}_{n_T}\right) \sim \NN \left(\boldsymbol{\mu}^{(0)}, {\Sigma}^{(0)}\right),
\end{equation}
where $n_T$ is the number of instances in the target task, $\boldsymbol{\mu}^{(0)}$ is modeled as a linear function of $p$ features $\bm x^{(0)}_1, \bm x^{(0)}_2, ..., \bm x^{(0)}_p$, and ${\Sigma}^{(0)}$ is the covariance matrix assumed to be known or estimable from independent data.
We focus on the high-dimensional regime, that is, $ n_T \ll p$.
Similarly, for the $K$ source tasks, we consider 
\begin{equation}\label{eq:yk}
    \bm{Y}^{(k)} = \left(Y^{(k)}_1, Y^{(k)}_2, \dots, Y^{(k)}_{n_S}\right) \sim \NN\left(\boldsymbol{\mu}^{(k)}, {\Sigma}^{(k)}\right), \quad \forall k \in [K] = \{ 1, 2, \dots, K\}.
\end{equation}
For the $k$-th source model, $\boldsymbol{\mu}^{(k)}$ is modeled as a linear function of $p$ features $\bm x^{(k)}_1, \bm x^{(k)}_2, ..., \bm x^{(k)}_p$ and ${\Sigma}^{(k)}$ is the known covariance matrix.
We assume, for simplicity, that each source task has the same sample size $n_S$.
The goal is to statistically test the significance of the relationship between the features and the response in the target task after applying transfer learning for high-dimensional regression.

\subsection{Transfer Learning for High-Dimensional Regression (TransFusion \cite{he2024transfusion})}\label{subsec_transfusion}

The procedure of TransFusion is described as follows:

\vspace{8pt}
\textbf{\textit{Step 1. Co-Training}}. This step makes use of both target and source samples to estimate the regression coefficients $\bm \beta^{(k)}$s and $\bm \beta^{(0)}$ by solving the optimization problem:

\[
\hat{\boldsymbol\beta} = \operatorname*{argmin}_{\boldsymbol\beta \in \mathbb{R}^{(K+1)p}} 
\left\{ 
\frac{1}{2N} \sum_{k=0}^{K} \| \bm{Y}^{(k)} - {X}^{(k)} \boldsymbol\beta^{(k)} \|_2^2 
+ \lambda_0 \left( \| \boldsymbol\beta^{(0)} \|_1 + \sum_{k=1}^{K} a_k \| \boldsymbol\beta^{(k)} - \boldsymbol\beta^{(0)} \|_1 \right) 
\right\},
\]
where $N = K n_S + n_T$, $\lambda_0$ is the hyper-parameter and $\left\{a_k\right\}_{k=1}^{K}$ are the weights associated with the $K$ source tasks.
As stated in Appendices A and D of the TransFusion paper \cite{he2024transfusion}, the above optimization problem is equivalent to solving:
\begin{equation}\label{eq:theta}
 \hat{\boldsymbol{\theta}} = \underset{\boldsymbol{\theta} \in \mathbb{R}^{{(K+1)p}}}{\operatorname{argmin}} \left\{ \frac{1}{2N}\| \bm Y - {X}\boldsymbol{\theta}\|^2_2 + \lambda_0 
\sum_{k=0}^Ka_{k}\|\boldsymbol{\theta}^{(k)}\|_1  \right\},   
\end{equation}
where $a_0 = 1$, $\bm Y = \big ( \bm Y^{(1)}, \bm Y^{(2)}, ..., \bm Y^{(K)}, \bm Y^{(0)} \big)^\top$,
\begin{align*}
 X =
\begin{pmatrix}
X^{(1)} & 0 & \cdots & 0 & {X}^{(1)} \\
0 & {X}^{(2)} & \cdots & 0 & {X}^{(2)} \\
\vdots & \vdots & \ddots & \vdots & \vdots \\
0 & 0 & \cdots & {X}^{(K)} & {X}^{(K)} \\
0 & 0 & \cdots & 0 & {X}^{(0)}
\end{pmatrix},
\quad
\boldsymbol{\theta} =
\begin{pmatrix}
\boldsymbol{\beta}^{(1)}-\boldsymbol{\beta}^{(0)} \\ 
\boldsymbol{\beta}^{(2)} -\boldsymbol{\beta}^{(0)}\\ 
\vdots \\
\boldsymbol{\beta}^{(K)}- \boldsymbol{\beta}^{(0)} \\
\boldsymbol{\beta}^{(0)}
\end{pmatrix}.
\end{align*}
The problem \eq{eq:theta} is a weighted LASSO problem and can therefore be directly solved using existing algorithms, such as proximal gradient descent.
After obtaining $\hat{\bm \theta}$, we can directly identify $\hat{\bm \beta}$, which is then used to compute the estimator
\begin{equation}\label{eq:w}
    \hat{\boldsymbol{w}} = \frac{n_S}{N} \sum_{k=1}^K \hat{\boldsymbol{\beta}}^{(k)} + \frac{n_T}{N} \hat{\boldsymbol{\beta}}^{(0)}.
\end{equation}
The motivation for computing the estimator $\hat{\bm w}$ can be found in \S2.1 of \cite{he2024transfusion} 

\vspace{8pt}
\textbf{\textit{Step 2. Local Debias}}. This step aims to refine the initial estimator $\hat{\bm w}$ and compute the final estimated coefficients for the target task:
\begin{align}
    & \hat{\boldsymbol{\delta}} = \underset{\boldsymbol{\delta} \in \mathbb{R}^p}{\operatorname{argmin}} \left\{ \frac{1}{2n_T} \left\| \bm Y^{(0)} - {X}^{(0)} \hat{\boldsymbol{w}} - {X}^{(0)} \boldsymbol{\delta} \right\|_2^2 + \tilde{\lambda} \|\boldsymbol{\delta}\|_1\right\}, \label{eq:delta} \\ 
    & \hat{\boldsymbol{\beta}}^{(0)}_{\rm TransFusion} = \hat{\boldsymbol{w}} + \hat{\boldsymbol{\delta}}  \label{eq:beta_transfusion}.
\end{align}
All the steps in TransFusion can be summarized in Algorithm \ref{algo_transfusion}.

\begin{algorithm}
\caption{TransFusion \cite{he2024transfusion}}\label{algo_transfusion}
\begin{algorithmic}[1]

\Require $\left({X}^{(0)}, \bm Y^{(0)}\right)$, $\left\{ ({X}^{(k)}, \bm Y^{(k)} )\right\}_{k=1}^{K}$, $\lambda_0$, $\tilde{\lambda}$ and $\left\{a_k\right\}_{k=1}^{K}$

\vspace{5pt}
\State Compute $\hat{\boldsymbol{\theta}} \gets$ Eq. \eq{eq:theta}

\vspace{5pt}
\State Calculate  $\hat{\boldsymbol{w}} \gets $ Eq. \eq{eq:w}

\vspace{5pt}
\State Obtain $\hat{\boldsymbol{\delta}} \gets$ Eq. \eq{eq:delta}

\vspace{5pt}
\State $\hat{\boldsymbol{\beta}}^{(0)}_{\rm TransFusion} \gets \hat{\boldsymbol{w}} + \hat{\boldsymbol{\delta}}$ $\quad$ // Eq. \eq{eq:beta_transfusion}

\vspace{5pt}
\Ensure $\hat{\boldsymbol{\beta}}^{(0)}_{\rm TransFusion}$
\end{algorithmic}
\end{algorithm}

\subsection{Feature Selection and Statistical Inference}\label{subsec_feature_selection_statistical_inference}

Since the TransFusion yields sparse solutions, the selected feature set in the target task is defined as:
\begin{align} \label{eq:selected_set}
	\cM = \left\{ j \in [p] : \left (\hat{\bm\beta}^{(0)}_{\rm TransFusion} \right)_j \neq 0 \right\}.
\end{align}
Our goal is to assess whether the features selected in \eq{eq:selected_set} are truly relevant or merely selected by chance. To perform the inference on the $j^{\rm th}$ selected feature, we consider the following statistical hypotheses:
\begin{equation} \label{eq:hypotheses}
    \text{H}_{0, j} : \beta_j = 0 
    \quad \text{vs.} \quad 
    \text{H}_{1, j} : \beta_j \neq 0,
\end{equation}
where $\beta_j = \left [ \Big ( {X^{(0)}_{\cM}}^\top X^{(0)}_{\cM} \Big )^{-1} {X^{(0)}_{\cM}}^\top \bm \mu^{(0)} \right ]_j$, ${X}^{(0)}_{\mathcal{M}}$ is the sub-matrix of ${X}^{(0)}$ made up of columns in the set ${\mathcal{M}}$.

A natural choice of test statistic for testing these hypotheses is the least squares estimate, defined as:
\begin{equation}\label{eq:test_statistic}
    \tau_j = 
    \left [ \Big ( {X^{(0)}_{\cM}}^\top X^{(0)}_{\cM} \Big )^{-1} {X^{(0)}_{\cM}}^\top \bm Y^{(0)} \right ]_j
    = \bm \eta_j^\top \bm Y,
\end{equation}
where $\bm Y = \big ( \bm Y^{(1)}, \bm Y^{(2)}, ..., \bm Y^{(K)}, \bm Y^{(0)} \big)^\top$ as defined in \eq{eq:theta} and $\bm \eta_j$ is the direction of the test statistic defined as:
\begin{align} \label{eq:eta}
\boldsymbol{\eta}_j = \begin{pmatrix}
\mathbf{0}_{Kn_S} \\
{X}^{(0)}_{\mathcal{M}} 
\left( {X}^{(0) \top}_{\mathcal{M}} 
{X}^{(0)}_{\mathcal{M}} \right)^{-1}\bm{e}_j
\end{pmatrix}.
\end{align}
Here,  $\bm 0_{Kn_S} \in \mathbb{R}^{Kn_S}$ denotes a vector of all zeros, $\bm e_j \in \mathbb{R}^{|\mathcal{M}|}$ is a basis vector with a 1 at the $j^{\text{th}}$ position and 0 elsewhere.

\subsection{Decision Making based on $p$-values}\label{subsec_valid_p_value}
After computing the test statistic in \eq{eq:test_statistic}, we proceed to calculate the corresponding $p$-value. Given a significance level $\alpha \in [0,1]$ (e.g., 0.05), we reject the null hypothesis $\text{H}_{0,j}$ and conclude that the $j^{\text{th}}$ feature is relevant if the $p$-value is less than or equal to $\alpha$. Conversely, if the $p$-value exceeds $\alpha$, there is insufficient evidence to support the relevance of the $j^{\rm th}$ selected feature.
\\ \\
\textbf{Challenge of computing a valid $p$-value. } The conventional (or naive) $p$-value is defined as:
\begin{equation*}
p_{j}^{\text{naive}} = \mathbb{P}_{\text{H}_{0,j}}\Big(
\left|\boldsymbol{\eta}_{j}^{\top} \bm Y
\right| \geq 
\left|\boldsymbol{\eta}_{j}^{\top}
\bm Y_{\rm obs}\right|
\Big),
\end{equation*}
where $\bm Y_{\rm obs}$ is an observation (realization) of the random vector $\bm Y$.
If the hypotheses in \eq{eq:hypotheses} are fixed in advance, i.e., non-random, the vector $\boldsymbol{\eta}_{j}$ is independent of both the data and the TransFusion algorithm. Consequently, the naive $p$-value is valid in the sense that
\begin{equation}\label{eq:naive}
\mathbb{P}\Big(
\underbrace{p_{j}^{\text{naive}} \leq \alpha \mid \text{H}_{0,j} \text{ is true}}_{\text{a false positive}}
\Big) = \alpha, \quad \forall\alpha\in[0,1],
\end{equation}
i.e., the false positive rate is controlled under a predefined significance level. 
However, in our setting, the vector $\boldsymbol{\eta}_{j}$ is influenced by both TransFusion and the data, as it is defined based on the set of features selected by applying TransFusion to the data.
As a result, the property of a valid $p$-value in \eqref{eq:naive} is no longer satisfied. Hence, the naive $p$-value is \emph{invalid}.


%
%
%

\section{Proposed SI Method}\label{sec3}
This section outlines our approach and presents the technical details for computing valid $p$-values in the TransFusion algorithm, thereby addressing the invalidity of naive $p$-values.

\begin{figure}[htbp]
    \centering
    \includegraphics[
        width=\textwidth,
        trim=0pt 2.2cm 0pt 2.5cm,
        clip
    ]{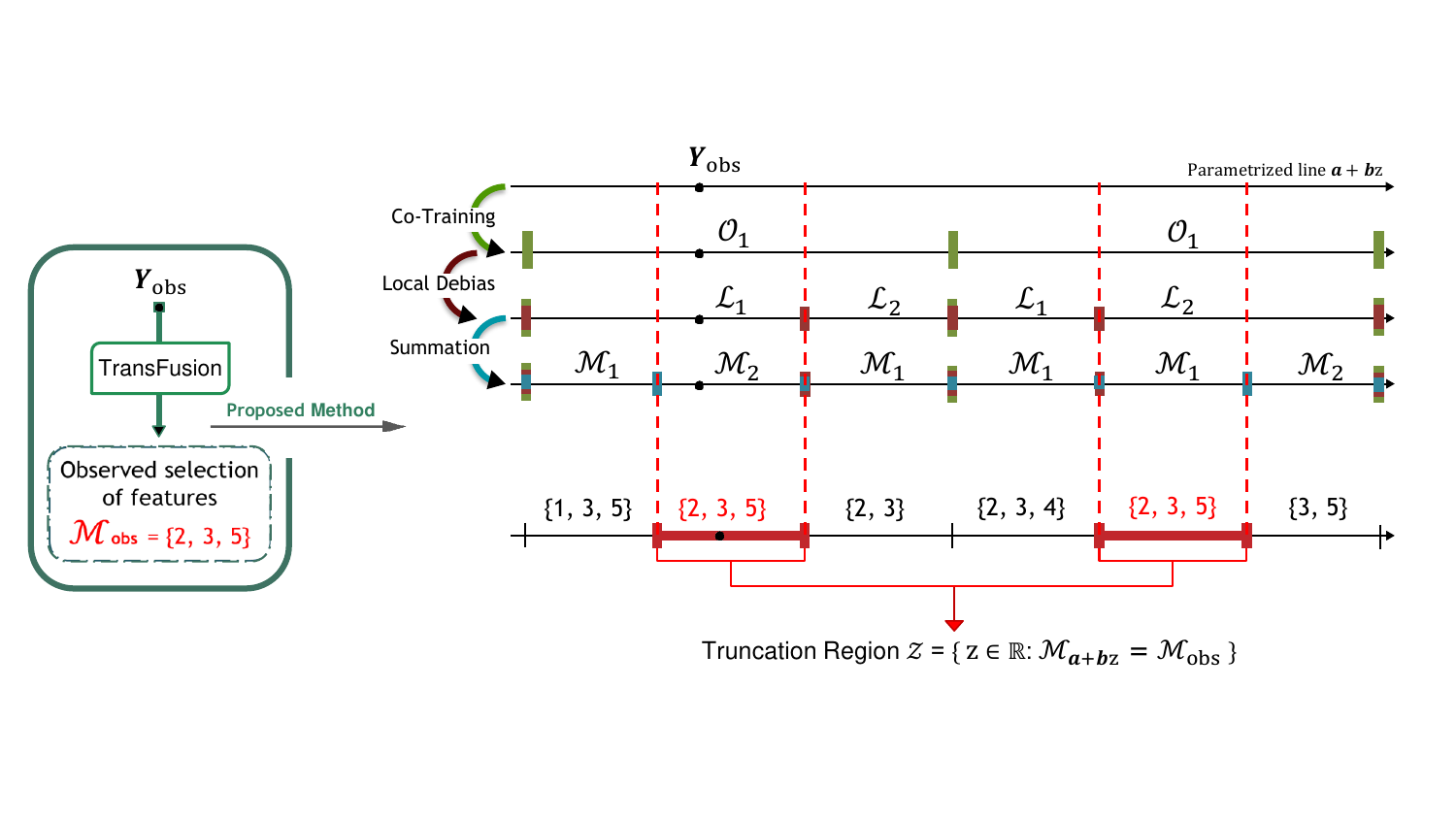}
    \caption{Illustration of the Selective Inference method tailored for TransFusion. First, the TransFusion algorithm is applied to the combined source and target data, yielding the target estimate $\hat{\boldsymbol{\beta}}_{\rm TransFusion}^{(0)}$ and identifying the selected feature set $\mathcal{M}_{\text{obs}}$. The data are then parameterized using a scalar parameter $z$ in the dimension of the test statistic to define the truncation region $\mathcal{Z}$, for which data yield the same feature selection results. To improve computational efficiency, a divide-and-conquer strategy is employed to effectively identify $\mathcal{Z}$. Finally, valid statistical inference is performed within the identified region $\mathcal{Z}$.}

    \label{fig:si_method}
\end{figure}

\subsection{The valid $p$-value in TransFusion}\label{subsec3.1}
To compute the valid $p$-value, we derive the sampling distribution of the test statistic in equation \eqref{eq:test_statistic} by leveraging the SI framework \cite{lee2016exact}, specifically by conditioning on the results of the TransFusion method:
\begin{equation*}
\mathbb{P} \left( \boldsymbol\eta_j^\top\bm{Y} \mid \mathcal{M}(\bm{Y}) = \mathcal{M}_{\text{obs}} \right),
\end{equation*}
where $\mathcal{M}(\bm{Y})$ denotes the set of features selected by applying TransFusion to any random vector $\bm Y$, and $\cM_{\rm obs} = \mathcal{M}(\bm{Y}_{\rm obs}) $ represents the observed set of selected features.
Next, based on the distribution of the test statistic derived above, we define the selective $p$-value as follows:
\begin{equation}\label{eq:p_selective}
p_{j}^{\text{selective}} = \mathbb{P}_{\text{H}_{0,j}}\Big(
\left|\boldsymbol{\eta}_{j}^{\top} \bm Y
\right|\geq 
\left|\boldsymbol{\eta}_{j}^{\top}
\bm{Y}_{\text{obs}}\right| \Big| \,\mathcal{E}
\Big),
\end{equation}
where the conditioning event $\mathcal{E}$ is defined as
\begin{equation}\label{eq:event}
\mathcal{E} = \left\{ \mathcal{M}(\bm{Y}) = \mathcal{M}_{\text{obs}}, \, \cQ(\bm Y) = \cQ_{\rm obs} \right \}. 
\end{equation}
The  $\cQ(\bm Y)$  is the sufficient statistic of the nuisance component, defined as:
\begin{equation}\label{eq:q}
\cQ(\bm Y) = \left( I_N - \bm b \boldsymbol{\eta}_j^\top \right) \bm Y
\end{equation}
where $\bm b = \Sigma {\bm \eta}_j \left( \boldsymbol{\eta}_j^\top \Sigma \boldsymbol{\eta}_j \right)^{-1}$, 
$
{\Sigma} =
\begin{pmatrix}
{\Sigma}^{(1)} & 0 & \cdots & 0 & 0\\
0 & {\Sigma}^{(2)} & \cdots & 0 & 0 \\
\vdots & \vdots & \ddots & \vdots & \vdots \\
0 & 0 & \cdots & {\Sigma}^{(K)} & 0 \\
0 & 0 & \cdots & 0 & {\Sigma}^{(0)}  
\end{pmatrix}
$,
and $\cQ_{\rm obs} = \cQ(\bm Y_{\rm obs})$.
\begin{remark}
In the SI framework, it is standard to condition on the sufficient statistic of the nuisance component to ensure tractable inference. 
In our setting, the  $\cQ(\bm Y)$ serves this role and aligns with the component $\boldsymbol{z}$ in the seminal paper of \cite{lee2016exact} (see Section 5, Equation (5.2)). This additional conditioning is commonly adopted across the SI literature, including all related works cited in this paper.
\end{remark}

\begin{lemma}\label{lemma_validity_of_p_selective}
The selective $p$-value proposed in $\eqref{eq:p_selective}$  satisfies the validity property:
\begin{equation*}
\mathbb{P}_{\text{H}_{0,j}} \left( p_j^{\text{selective}} \leq \alpha \right) = \alpha, \quad \forall \alpha \in [0, 1].
\end{equation*}
\begin{proof}
The proof is deferred to Appendix \red{\ref{proof:lemma_validity_of_p_selective}}.
\end{proof}
\end{lemma}

\noindent
Lemma \ref{lemma_validity_of_p_selective} establishes that the proposed selective $p$-value ensures the theoretical control of the FPR for any significance level $\alpha \in [0, 1]$. 
Once the event $\cE$ in equation \eqref{eq:event} is identified, the selective $p$-value can be computed. 
The detailed characterization of $\cE$ will be provided in the following section.

\subsection{Characterization of the Conditioning Event $\cE$}\label{subsec3.2}
We define the set of $\bm Y \in \mathbb{R}^{N}$ that satisfies the conditions in \eqref{eq:event} as:

\begin{equation}\label{eq:Y}
\mathcal{Y} = 
\Big \{ \bm Y \in \mathbb{R}^N \mid \mathcal{M}(\bm Y) = \mathcal{M}_{\rm obs}, \, \cQ(\bm Y) = \cQ_{\rm obs} 
\Big \}.
\end{equation}
As stated in the following lemma, the subspace $\mathcal{Y}$ lies along a line within $\mathbb{R}^{N}$.
\begin{lemma}\label{lemma_line}
The conditional data space $\mathcal{Y}$ in \eqref{eq:Y} is restricted to a line parametrized by a scalar parameter $z \in \mathbb{R}$ as follows:
\begin{equation}\label{eq:Yz}
    \mathcal{Y} = \{ \bm Y (z) = \bm a + \bm b z \mid z \in \mathcal{Z} \},
\end{equation}
where $\bm{a} = \cQ_\text{obs}$, $\bm b$ is defined 
in \eqref{eq:q}, and
\begin{equation}\label{eq:Z}
    \mathcal{Z} = \{ z \in \mathbb{R} \mid \mathcal{M}(z) = \mathcal{M}_\text{obs} \}.
\end{equation}
Here, with a slight abuse of notation, $\cM(z)$ is equivalent to $\cM \big (\bm Y(z) \big )$.
\end{lemma}

\begin{proof}
The proof is deferred to Appendix \red{\ref{proof:lemma_line}}.
\end{proof}

\begin{remark}
Lemma \ref{lemma_line} demonstrates that the data space, initially in $\mathbb{R}^{N}$, can be reduced to the one-dimensional subspace $\mathcal{Z}$ defined in \eqref{eq:Z}. This reduction, which confines the data to a line, was implicitly leveraged in \cite{lee2016exact} and later explicitly addressed in \S6 of \cite{liu2018more}.
\end{remark}

\textbf{Reformulating selective $p$-value computation in terms of $\cZ$. }  
We define the random variable $\cZ$ along with its observed value $\cZ_{\rm obs}$ as follows:
\begin{align*}
Z = \boldsymbol\eta_j^\top \bm Y \in \mathbb{R}\quad \text{and} \quad Z_{\text{obs}} = \boldsymbol\eta_j^\top \bm{Y}_{\text{obs}} \in \mathbb{R}.
\end{align*}
Using this notation, the selective $p$-value from equation \eqref{eq:p_selective} can be reformulated as:

\begin{equation}\label{eq:p_selective_reformulated}
	p_j^{\text{selective}} = \mathbb{P}_{\mathrm{H}_{0,j}} \Big( |Z| \geq |Z_{\text{obs}}| ~\big|~ Z \in \mathcal{Z} \Big).
\end{equation}
Since $\bm Y$ is normally distributed, $Z \mid Z \in \mathcal{Z}$ follows a truncated normal distribution.
Identifying the truncation region $\cZ$ is the key remaining step, as it enables straightforward computation of the selective $p$-value in equation \eq{eq:p_selective_reformulated}.

\subsection{Identification of Truncation Region $\mathcal{Z}$}\label{subsec3.3}

At first glance, identifying $\cZ$ appears \emph{intractable}, as it requires applying the TransFusion algorithm to $\bm Y (z) = \bm a + \bm b z$ for \emph{infinitely} many values of $z$ in order to obtain the set of selected features  $\cM (z)$, and check whether it matches the observed set of selected features $\cM_{\rm obs}$.
However, this seemingly intractable problem can be reduced to finite and efficiently computable sets of constraints that exactly characterize $\cZ$.
Inspired by \cite{duy2022more} and \cite{le2024cad}, we adopt a ``divide-and-conquer'' strategy and propose a method (illustrated in Fig. \ref{fig:si_method}) to efficiently identify $\cZ$, described as follows:

\begin{itemize}
  \item We decompose the problem into multiple sub-problems by introducing additional conditioning based on the execution of the TransFusion algorithm.

%
  \item We show that each subproblem can be efficiently solved, as it is characterized by a finite number of linear inequalities.
  
  \item We combine multiple sub-problems to obtain $\mathcal{Z}$.
\end{itemize}

\vspace{3pt}
\textbf{Divide-and-conquer strategy.} 
Let us define the active sets (i.e., sets of features with non-zero coefficients) obtained from the optimization problems \eqref{eq:theta} and \eqref{eq:delta} when the TransFusion is applied to $\bm Y (z)$ as follows:
\begin{align} \label{eq:cO_cL}
	\cO(z) =  \left\{ j^\prime: \hat{\theta}_{j^\prime} (z) \neq 0 \right \},\quad
	\quad 
	\cL(z) = \left\{ j^{\prime\prime}: \hat{\delta}_{j^{\prime\prime}} (z) \neq 0 \right \},
\end{align}
where $j^\prime \in [(K + 1)p]$ and $j^{\prime\prime} \in [p]$.
We define $\cS_{\cO(z)}$ and $\cS_{\cL(z)}$ as the sets of coefficient signs corresponding to the selected features in $\cO(z)$ and $\cL(z)$, respectively.
The entire one-dimensional space $\mathbb{R}$ can be decomposed as:
\begin{align*}
\mathbb{R} = 
\bigcup_{u \in [U]} \bigcup_{v \in [V_u]} \bigcup_{t \in [T_{u, v}]}
\underbrace{\left\{ z \in \mathbb{R} ~ \middle| ~\begin{array}{l}
\cO(z) = \cO_u, \, \cS_{\cO(z)} =  \cS_{\cO_u},\\
\cL(z) = \cL_v, \, \cS_{\cL(z)} =  \cS_{\cL_v},\\
\cM(z) = \cM_t, \, \cS_{\cM(z)} = \cS_{\cM_t}, 
\end{array}\right\}}_{\text{
a sub-problem of additional conditioning}},
\end{align*}
where $U$ denotes the total number of possible combinations of active sets $\cO(z)$ and their corresponding signs along the line, $V$ represents the number of all possible combinations of active sets $\cL(z)$ and their signs, given that $\bm \theta(z)$ has the active set $\cO_u$; and $T_{u, v}$ denotes the number of all possible combinations of active sets $\cM(z)$ and their signs, given $\cO_u$ and $\cL_v$.
For $u \in [U], \, v \in [V_u], \,t \in [T_{u, v}]$, our goal is to search a set
\begin{equation}\label{eq:A}
    \mathcal{A} = \Big \{(u,v,t) : \mathcal{M}_t = \mathcal{M}_{\text{obs}} \Big \}.
\end{equation}
After obtaining $\cA$, the truncation region $\mathcal{Z}$ in $\eqref{eq:Z}$ can be obtained as follows:
\begin{align} \label{eq:cZ_union}
    \mathcal{Z} 
    & = 
    \big \{ z \in \mathbb{R} \mid \mathcal{M}(z) = \mathcal{M}_\text{obs} \big \} \nonumber \\
    & = 
    \bigcup_{{(u,v,t) \in \mathcal{A}}}
    \left\{ z \in \mathbb{R}  ~ \middle | ~
    \begin{array}{l}
    \cO(z) = \cO_u, \, \cS_{\cO(z)} =  \cS_{\cO_u},\\
\cL(z) = \cL_v, \, \cS_{\cL(z)} =  \cS_{\cL_v},\\
\cM(z) = \cM_t, \, \cS_{\cM(z)} = \cS_{\cM_t},
    \end{array}\right\}
\end{align}

\textbf{Solving of each sub-problem. } For any $u \in [U]$, $v \in [V_u]$ and $t \in [T_{u, v}]$, we define the subset of the one-dimensional projected data space along the line for the subproblem as follows:
\begin{align} \label{eq:cZ_uvt}
\mathcal{Z}_{u,v,t} = 
\left\{ z \in \mathbb{R}  ~ \middle | ~
    \begin{array}{l}
    \cO(z) = \cO_u, \, \cS_{\cO(z)} =  \cS_{\cO_u},\\
    \cL(z) = \cL_v, \, \cS_{\cL(z)} =  \cS_{\cL_v},\\
    \cM(z) = \cM_t, \, \cS_{\cM(z)} = \cS_{\cM_t},
    \end{array}\right\},
\end{align}

The sub-problem region $\mathcal{Z}_{u,v,t}$ can be re-written as:
\[
\mathcal{Z}_{u,v,t} = \mathcal{Z}_u \,\cap \mathcal{Z}_v \,\cap \mathcal{Z}_t \,, 
\]
where
\begin{align}
    \mathcal{Z}_u &= \{ z \in \mathbb{R} \mid \cO(z) = \cO_u, \, \cS_{\cO(z)} =  \cS_{\cO_u} \}, \label{eq:cZu} \\ 
    \mathcal{Z}_v &= \{ z \in \mathbb{R} \mid \cL(z) = \cL_v, \, \cS_{\cL(z)} =  \cS_{\cL_v} \}, \label{eq:cZv} \\
    \mathcal{Z}_t &= \{ z \in \mathbb{R} \mid \cM(z) = \cM_t, \, \cS_{\cM(z)} = \cS_{\cM_t} \} \label{eq:cZt}.
\end{align}

\begin{lemma}\label{lem:Zu}
The set $\mathcal{Z}_u$ can be characterized by a set of linear inequalities w.r.t. $z$:

\begin{equation}
\mathcal{Z}_u = \{ z \in \mathbb{R} \mid \boldsymbol\psi z \leq \boldsymbol\gamma \} ,
\end{equation}
where the vectors $\boldsymbol{\psi}$ and $\boldsymbol{\gamma}$ are defined in Appendix \red{\ref{proof:lemZu}}.

\begin{proof}
The proof is deferred to Appendix \red{\ref{proof:lemZu}}.
\end{proof}

\end{lemma}
\begin{lemma}\label{lem:Zv}
The set $\mathcal{Z}_v$ can be characterized by a set of linear inequalities w.r.t. $z$:

\begin{equation}
\mathcal{Z}_v = \{ z \in \mathbb{R} \mid \boldsymbol\nu z \leq \boldsymbol\kappa \},
\end{equation}
where the vectors $\boldsymbol{\nu}$ and $\boldsymbol{\kappa}$ are defined in Appendix \red{\ref{proof:lemZv}}.
\end{lemma}

\begin{proof}
The proof is deferred to Appendix \red{\ref{proof:lemZv}}.
\end{proof}

\begin{lemma}\label{lem:Zt}
The set $\mathcal{Z}_t$ can be characterized by a set of linear inequalities w.r.t. $z$:

\begin{equation}
\mathcal{Z}_t = \{ z \in \mathbb{R} \mid \boldsymbol\omega z \leq \boldsymbol \rho \},
\end{equation}
where the vectors $\boldsymbol{\omega}$ and $\boldsymbol{\rho}$ are defined in Appendix \red{\ref{proof:lemZt}}.
\end{lemma}

\begin{proof}
The proof is deferred to Appendix \red{\ref{proof:lemZt}}.
\end{proof}

Lemmas~\ref{lem:Zu}, ~\ref{lem:Zv} and ~\ref{lem:Zt} guarantee that selected features and coefficient signs for $\hat{\boldsymbol{\theta}}(z)$, $\hat{\boldsymbol{\delta}}(z)$, and $\hat{\boldsymbol{\beta}}^{(0)}_{\rm TransFusion} (z)$ remain unchanged when applying the TransFusion algorithm on any $z \in \mathcal{Z}_{u,v,t}$. 
They also indicate that $\mathcal{Z}_u$, $\mathcal{Z}_v$ and $\mathcal{Z}_t$  can be \textit{analytically obtained} by solving the systems of 
linear inequalities. Once $\mathcal{Z}_u$, $\mathcal{Z}_v$  
and $\mathcal{Z}_t$ are computed, the sub-problem region 
$\mathcal{Z}_{u,v,t}$ in $\eq{eq:cZ_uvt}$ is obtained by 
$\mathcal{Z}_{u,v,t} = \mathcal{Z}_u \cap \mathcal{Z}_v \cap \mathcal{Z}_t$.\\

\textbf{Combining multiple sub-problems. } To identify $\mathcal{A}$ in $\eqref{eq:A}$, the TransFusion algorithm is repeatedly applied to a sequence of datasets $\bm Y(z)$, within sufficiently wide range of $z \in [z_{\min}, z_{\max}]$\footnote{We set $z_{\min} = -20\sigma$ and $z_{\max} = 20\sigma$, $\sigma$ is the standard deviation of the distribution of the test statistic, because the probability mass outside this range is negligibly small.}. Since $\mathcal{Z}_u$, $\mathcal{Z}_v$ and $\mathcal{Z}_t$ are intervals, $\mathcal{Z}_{u,v,t}$ is an interval. We denote $\mathcal{Z}_u = [l_u, r_u]$, ${\cal Z}_{u,v}= \mathcal{Z}_u \cap \mathcal{Z}_v = [l_{u,v},r_{u,v}]$ and ${\cal Z}_{u,v, t} = [l_{u,v,t},r_{u,v,t}]$. The divide-and-conquer procedure can be summarized in Algorithm $\ref{alg:d-a-c}$. After obtaining ${\mathcal A}$ by Algorithm $\ref{alg:d-a-c}$. We can compute ${\mathcal Z}$ in $\eqref{eq:cZ_uvt}$, which is subsequently used to obtain the proposed selective $p$-value in $\eqref{eq:p_selective_reformulated}$. The entire steps of the proposed TransFusion SI method are summarized in Algorithm $\ref{alg:SI}$.

\begin{algorithm}[!t]
\caption{\texttt{divide\_and\_conquer}}\label{alg:d-a-c}
\begin{algorithmic}[1]
{\small
\Require $ X, \bm a, \bm b,  z_{\min}, z_{\max}$

\vspace{2pt}
\State \text{Initialization:} $u = 1$, $v = 1$, $t = 1$, $z = z_{\min}$, $\mathcal{A} = \emptyset$

\vspace{2pt}
\While{$z < z_{\max}$}

    
    \vspace{2pt}
    \State $\cO_u$ and $ \cS_{\cO_u}$  $\leftarrow$ Solving Eq. $\eqref{eq:theta}$  with $X$ and $ \bm Y (z) = \bm a + \bm b z$
    
    \vspace{2pt}   
    \State $[l_u,\,r_u] = \mathcal{Z}_u \leftarrow$ Lemma $\ref{lem:Zu}$
    
    \vspace{2pt}
    \State $r_{u,v} = l_u$
    
    \vspace{2pt}
    \While{$r_{u,v} < r_u$}
        
        \vspace{2pt}
        \State  $\cL_v$ and $ \cS_{\cL_v} \leftarrow$ Solving Eq. \eqref{eq:delta} with $X$ and $ \bm Y (z)$
        
        
        \vspace{2pt}
        \State $[l_{u,v}, r_{u,v}] = \mathcal{Z}_{u,v} \leftarrow \mathcal{Z}_{u} \cap \mathcal{Z}_{v}$ where $\mathcal{Z}_v \leftarrow$ Lemma $\ref{lem:Zv}$
        
        \vspace{2pt}
        \State $r_{u,v,t} = l_{u,v}$
        
        \vspace{2pt}
        \While{$r_{u,v,t} < r_{u,v}$}
        	
	         \vspace{2pt}
            	\State $\mathcal{M}_t$ and $\cS_{\cM_t} \leftarrow$ TransFusion on $X$ and $\bm Y (z)$
	
		
		\vspace{2pt}
	        \State $[l_{u,v,t}, r_{u,v,t}] \leftarrow \mathcal{Z}_{u,v}\cap \mathcal{Z}_t$ where $\mathcal{Z}_t \leftarrow$ Lemma $\ref{lem:Zt}$
	        
	        \vspace{4pt}
	        \State $\mathcal{A} \leftarrow \mathcal{A} \cup \{(u,v,t)\}$ if $\mathcal{M}_t = \mathcal{M}_{\text{obs}}$
	        
	       \vspace{2pt}
	        \State $t \leftarrow t + 1$, $z = r_{u,v,t}$
	        \vspace{2pt}
        \EndWhile
            \State $t \leftarrow 1$, $v \leftarrow v + 1$, $z = r_{u,v,t}$

    \EndWhile
    \vspace{2pt}
    \State $v \leftarrow 1$, $u \leftarrow u + 1$, $z = r_{u,v,t}$
    \vspace{2pt}
\EndWhile
\vspace{2pt}
\Ensure $\mathcal{A}$
}
\end{algorithmic}
\end{algorithm}

\begin{algorithm}[!t]
\caption{\texttt{PTL-SI}}\label{alg:SI}
\begin{algorithmic}[1]
{\small
\Require $\big (X^{(0)}, \, \bm Y^{(0)}_{\text{obs}}\big), \, \Big\{\big(X^{(k)}, \, \bm Y^{(k)}_{\text{obs}}\big)\Big\}_{k=1}^K, \,z_{\min}, \, z_{\max}$

\vspace{3pt}
\State Construct $X$ and $\bm Y_\text{obs}$ $ \leftarrow$ Eq. $\eqref{eq:theta}$


\vspace{3pt}
\State $\mathcal{M}_{\text{obs}} \leftarrow$ TransFusion on $\left( X,\, \bm Y_{\text{obs}}\right)$

\vspace{3pt}
\For{$j \in \mathcal{M}_{\text{obs}}$}

    \vspace{3pt}
    \State Compute $\boldsymbol\eta_j \leftarrow$ Eq. $\eqref{eq:eta}$, $\bm a$ and $\bm b \leftarrow$ Eq. $\eqref{eq:Yz}$
    
    \vspace{4pt}
    \State $\mathcal{A} \leftarrow \texttt{divide\_and\_conquer}$ $( X, \bm a, \, \bm b, \,  z_{\min}, z_{\max}$)
    
    \vspace{4pt}
    \State Identify $\mathcal{Z} \leftarrow$ Eq. $\eqref{eq:cZ_union}$ with $\mathcal{A}$
    
    \vspace{4pt}
    \State Compute $p_j^{\text{selective}} \leftarrow$ Eq. $\eqref{eq:p_selective_reformulated}$ with $\mathcal{Z}$
    \vspace{3pt}
\EndFor
\vspace{3pt}
\Ensure $\{p_j^{\text{selective}} \}_{j \in \mathcal{M}_{\text{obs}}}$
}
\end{algorithmic}
\end{algorithm}

\section{Extension to Oracle Trans-Lasso \cite{li2022transfer}} \label{sec:extension}

Given $K$ informative auxiliary samples (i.e., source tasks)
$ 
\left \{ 
	(X^{(k)}, \bm Y^{(k)})
\right \}_{k = 1}^K 
$, the procedure of Oracle Trans-Lasso \cite{li2022transfer} is detailed as follows:

\vspace{8pt}
\textbf{\textit{Step 1.}} Compute
\begin{align*}
	\hat{\bm w}^{\cI}
	= 
	\underset{\boldsymbol{w} \in \mathbb{R}^{{p}}}{\operatorname{argmin}}
	\left\{ 
	\frac{1}{2n_{\mathcal{I}}} 
	\sum \limits_{k \in \cI} 
	\| \bm{Y}^{(k)} -  X^{(k)} \boldsymbol{w} \|_2^2 + \lambda_{\bm w} \| \boldsymbol{w} \|_1 
	\right\},
\end{align*}
where $\cI = \{1, 2, \dots, K \}$, $n_\mathcal{I} = |\mathcal{I}|n_S$, and $\lambda_{\bm w}$ is the hyper-parameter. 
The above optimization problem can be rewritten as follows:
\begin{align}\label{eq:w_OTL}
\hat{\boldsymbol{w}}^{\cI} = 
\underset{\boldsymbol{w} \in \mathbb{R}^{{p}}}{\operatorname{argmin}} \left\{ \frac{1}{2n_{\mathcal{I}}} \| \bm{Y}^{\mathcal{I}} -  X^{\mathcal{I}} \boldsymbol{w} \|_2^2 + \lambda_w \| \boldsymbol{w} \|_1 \right\}  
\end{align}
where 
\[ \bm{Y}^\mathcal{I} = \left(\begin{array}{c}
    \bm{Y}^{(1)} \\
    \bm{Y}^{(2)} \\
    \vdots \\
    \bm{Y}^{(K)} 
    \end{array}\right), \quad X^\mathcal{I} = \left(\begin{array}{c}
    X^{(1)} \\
    X^{(2)} \\
    \vdots \\
    X^{(K)} 
    \end{array}\right).
\]

\vspace{8pt}
\textbf{\textit{Step 2.}} Compute
\begin{align}
    & \hat{\boldsymbol{\delta}}^{\cI} = \underset{\boldsymbol{\delta} \in \mathbb{R}^p}{\operatorname{argmin}} \left\{ \frac{1}{2n_T} \left\| \bm Y^{(0)} - {X}^{(0)} \hat{\boldsymbol{w}}^{\cI} - {X}^{(0)} \boldsymbol{\delta} \right\|_2^2 + \lambda_{\bm \delta} \|\boldsymbol{\delta}\|_1\right\}, \label{eq:delta_OTL} \\ 
    & \hat{\boldsymbol{\beta}}^{(0)}_{\rm OTL} = \hat{\boldsymbol{w}}^{\cI} + \hat{\boldsymbol{\delta}}^{\cI}  \label{eq:beta_OTL},
\end{align}
where $\lambda_{\bm \delta}$ is the hyper-parameter.

\vspace{5pt}
We extend our proposed PTL-SI method to the case of Oracle Trans-Lasso.
Let us define the active sets  obtained from the optimization problems \eqref{eq:w_OTL},  \eqref{eq:delta_OTL}, and  \eqref{eq:beta_OTL} when the Oracle Trans-Lasso is applied to $\bm Y (z)$ as follows:
\begin{align} \label{eq:cO_cL_otl}
	\cO^{\rm otl}(z) &=  \left\{ j^\prime \in [p]: \hat{w}^{\cI}_{j^\prime} (z) \neq 0 \right\}, \\
	\cL^{\rm otl}(z) &= \left\{ j^{\prime\prime} \in [p] : \hat{\delta}^{\cI}_{j^{\prime\prime}} (z) \neq 0 \right \}, \\
	\cM^{\rm otl}(z) &= \left\{ j \in [p]: 
	\Big ( 
	\hat{\boldsymbol{\beta}}^{(0)}_{\rm OTL}
	\Big )_j
	 \neq 0 \right \}.
\end{align}
The truncation region in the case of Oracle Trans-Lasso is defined as follows:
\begin{align} \label{eq:cZ_union_otl}
    \mathcal{Z}^{\rm otl} 
    & = 
    \big \{ z \in \mathbb{R} \mid \mathcal{M}^{\rm otl}(z) = \mathcal{M}^{\rm otl}_\text{obs} \big \} \nonumber \\
    & = 
    \bigcup_{{(u,v,t) \in \cA^{\rm otl}}}
    \left\{ z \in \mathbb{R}  ~ \middle | ~
    \begin{array}{l}
    \cO^{\rm otl}(z) = \cO^{\rm otl}_u, \, \cS_{\cO^{\rm otl}(z)}=  \cS_{\cO^{\rm otl}_u},\\
\cL^{\rm otl}(z) = \cL^{\rm otl}_v, \, \cS_{\cL^{\rm otl}(z)}=  \cS_{\cL^{\rm otl}_v},\\
\cM^{\rm otl}(z) = \cM^{\rm otl}_t, \, \cS_{\cM^{\rm otl}(z)}= \cS_{\cM^{\rm otl}_t},
    \end{array}\right\},
\end{align}
where $ \mathcal{A}^{\rm otl} = \Big \{(u,v,t) : \mathcal{M}^{\rm otl}_t = \mathcal{M}^{\rm otl}_{\text{obs}} \Big \}$ that is similarly defined as Eq. \eq{eq:A}.

\begin{lemma}\label{lem:Z-OTL}
The sub-problem region $\cZ^{\rm otl}_{u,v,t}$ can be re-written as:
\begin{align} 
\cZ^{\rm otl}_{u,v,t}
     & = 
    \left\{ z \in \mathbb{R}  ~ \middle | ~
    \begin{array}{l}
    \cO^{\rm otl}(z) = \cO^{\rm otl}_u, \, \cS_{\cO^{\rm otl}(z)}=  \cS_{\cO^{\rm otl}_u},\\
\cL^{\rm otl}(z) = \cL^{\rm otl}_v, \, \cS_{\cL^{\rm otl}(z)}=  \cS_{\cL^{\rm otl}_v},\\
\cM^{\rm otl}(z) = \cM^{\rm otl}_t, \, \cS_{\cM^{\rm otl}(z)}= \cS_{\cM^{\rm otl}_t},
    \end{array}\right\}, \label{eq:sub_problem_region_otl} \\
    & = \mathcal{Z}^{\rm otl}_u \cap \mathcal{Z}^{\rm otl}_v \cap \mathcal{Z}^{\rm otl}_t,
\end{align}
where  the set $\mathcal{Z}_u^{\rm otl}$, $\mathcal{Z}_v^{\rm otl}$ and $\mathcal{Z}_t^{\rm otl}$ can be characterized by the sets of linear inequalities with respect to $z$:

\begin{equation*}
\mathcal{Z}^{\rm otl}_u = \{ z \in \mathbb{R} \mid \boldsymbol\psi^{\rm otl} z \leq \boldsymbol\gamma^{\rm otl} \} ,
\end{equation*}
\begin{equation*}
\mathcal{Z}^{\rm otl}_v = \{ z \in \mathbb{R} \mid \boldsymbol\nu^{\rm otl} z \leq \boldsymbol\kappa^{\rm otl} \} ,
\end{equation*}
\begin{equation*}
\mathcal{Z}^{\rm otl}_t = \{ z \in \mathbb{R} \mid \boldsymbol\omega^{\rm otl} z \leq \boldsymbol\rho^{\rm otl} \}.
\end{equation*}
The vectors $\boldsymbol{\psi}^{\rm otl}$,  $\boldsymbol{\gamma}^{\rm otl}, \boldsymbol{\nu}^{\rm otl}, \boldsymbol{\kappa}^{\rm otl},  \boldsymbol{\omega}^{\rm otl},  \boldsymbol{\rho}^{\rm otl}$ are defined in Appendix~\ref{proof:Z-OTL}.
\end{lemma}

\begin{proof}
The proof is deferred to Appendix~\ref{proof:Z-OTL}
\end{proof}

\section{Experiments}\label{sec5}
We demonstrate the performance of PTL-SI. Here, we present the main results.

\subsection{Experimental Setup}\label{subsec5.1}

\textbf{Methods for comparison.} We compared the performance of the following methods:
\begin{itemize}
    \item $\texttt{PTL-SI}$: proposed method,
    \item $\texttt{PTL-SI-oc}$: proposed method, which considers only one sub-problem, i.e., over-conditioning, described in \S \ref{subsec3.3} (extension of Lee et al. (2016) to our setting),
    \item $\texttt{DS}$: data splitting,
    \item $\texttt{Bonferroni}$: the most popular multiple testing,
    \item $\texttt{Naive}$: traditional statistical inference,
    \item $\texttt{No inference}$: TransFusion without inference.
\end{itemize}
We note that if a method fails to control the FPR at $\alpha$, it is \textit{invalid}, and its TPR becomes irrelevant. A method with a high TPR implies a low FNR.

\subsection{Numerical Experiments}\label{subsec5.2}
\textbf{Synthetic data generation.} We generated $\bm{Y}^{(0)}$ with $\bm{Y}^{(0)}_i = \left(X_i^{(0)}\right)^\top \boldsymbol{\beta}^{(0)} + \boldsymbol\varepsilon^{(0)}_i$, $X_i^{(0)} \sim \mathbb{N}(\mathbf{0}, I_p)$, and $\boldsymbol\varepsilon^{(0)}_i \sim \mathbb{N}(0, 1)$ $\forall i \in [n_T]$. Similarly, $\bm{Y}^{(k)}$ is generated with $\bm{Y}^{(k)}_j = \left(X_j^{(k)}\right)^\top \boldsymbol{\beta}^{(k)}  + \boldsymbol\varepsilon^{(k)}_j$, in which $X_j^{(k)} \sim \mathbb{N}(\mathbf{0}, I_p)$, and $\boldsymbol\varepsilon^{(k)}_j \sim \mathbb{N}(0, 1)$ $\forall j \in [n_S]$, $\forall k \in [K]$. We set $p = 300$, $n_S = 100$ and $\alpha = 0.05$. \\

For TransFusion, we partition $K$ sources into the informative auxiliary set $\cI$ and the non-informative auxiliary set $\cI^c$.  For the FPR experiments, all elements of $\boldsymbol{\beta}^{(0)}$ were set to 0. For the TPR experiments, the first $5$ elements of $\boldsymbol{\beta}^{(0)}$ were set to $\Gamma$. In all experiments, we set $\boldsymbol{\beta}^{(k)}_1 = -\Gamma$, $\boldsymbol{\beta}^{(k)}_i = \Gamma$, $\forall{i} \in \{2, 3, 4, 5\}$, and $\boldsymbol{\beta}^{(k)}_j = 0$, $\forall{j} \in \{6, 7, \dots, p\}$, $\forall{k} \in [K]$, then
$\boldsymbol{\beta}^{(k)}_i=\boldsymbol{\beta}^{(k)}_i + \Phi^{(k)}_i \sim \mathbb{N}(0, \,\Upsilon \times 0.5)$, $\forall i \in \{1, 2, \dots, 25\}$ if $k \in \mathcal{I}$, and otherwise, $\boldsymbol{\beta}^{(k)}_i=\boldsymbol{\beta}^{(k)}_i + \Phi^{(k)}_i \sim \mathbb{N}(0, \,\Upsilon \times 0.5 \times 10)$, $\forall i \in \{1, 2, \dots, 50\}$. Here, $\Upsilon$ is a tuning parameter that controls the noise level added to $\boldsymbol{\beta}^{(k)}$. Inference is conducted only on the target data, so $\boldsymbol{\beta}^{(k)}$ does not affect the results. We set $n_T = 50$, $|\cI| = 3$, $|\cI^c| = 2$, $\Gamma = 0.5$, $\Upsilon = 0.01$, for each experiment, we vary one of these variables while keeping the others fixed. Each experiment is repeated 1000 times.
For Oracle TransLasso, we set $|\cI| = 3$, $n_T \in \{40, 50, 60, 70\}$  and generate $\bm \beta^{(k)}$s, $k \in \{0, 1, \dots, K\}$ similarly to TransFusion with $\Gamma = 0.5$, $\Upsilon = 0.01$.\\

\textbf{The results of FPRs and TPRs.} The results of FPR and TPR in the case of TransFusion in multiple settings are shown in Figs.~\ref{fig:tpr_fpr_nt},~\ref{fig:tpr_fpr_beta},~\ref{fig:tpr_fpr_noise},~\ref{fig:tpr_fpr_informative} and~\ref{fig:tpr_fpr_uninformative}.
In the plots on the left, the $\texttt{PTL-SI}$, $\texttt{PTL-SI-oc}$, $\texttt{Bonferroni}$, and $\texttt{DS}$ controlled the FPR, whereas the $\texttt{Naive}$ and $\texttt{No Inference}$ \textit{could not}. Because the $\texttt{Naive}$ and $\texttt{No Inference}$ failed to control the FPR, we no longer considered their TPRs. In the plots on the right, the $\texttt{PTL-SI}$ has the highest TPR compared to other methods in all cases, i.e., the $\texttt{PTL-SI}$ has the lowest FNR. Fig.~\ref{fig:tpr_fpr_nt_otl} presents the corresponding results for Oracle Trans-Lasso, our proposed extension.\\

\begin{figure}[!t]
    \centering
    \begin{subfigure}{0.48\textwidth}
        \centering
        \includegraphics[width=\linewidth]{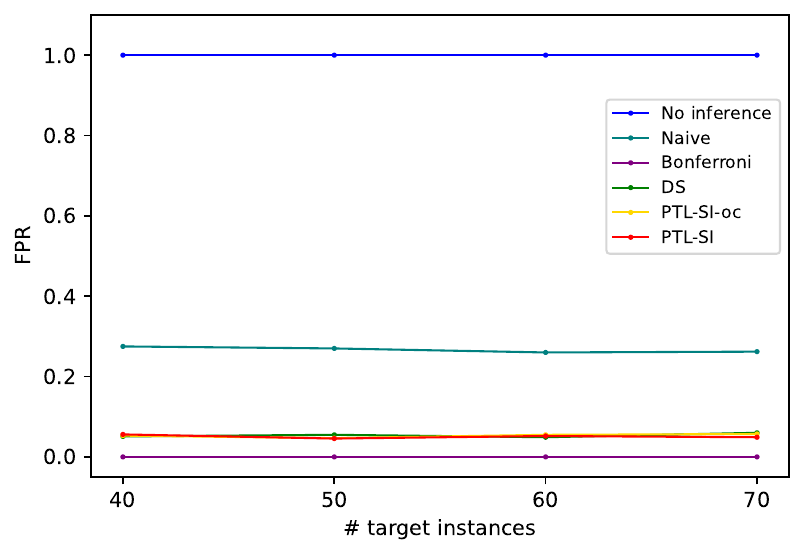}
        \caption{FPR}
        \label{fig:fpr_nt}
    \end{subfigure}\hfill
    \begin{subfigure}{0.48\textwidth}
        \centering
        \includegraphics[width=\linewidth]{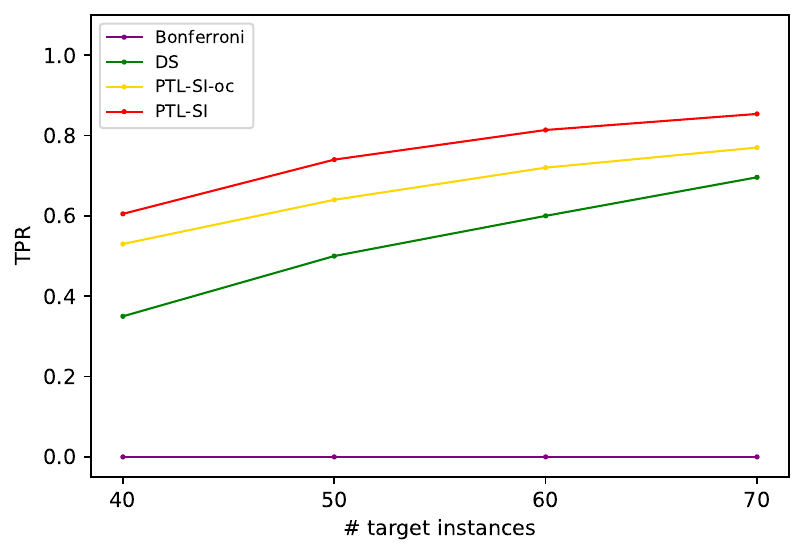}
        \caption{TPR}
        \label{fig:tpr_nt}
    \end{subfigure}
    \caption{FPR and TPR w.r.t. the number of target instances $n_T$}
    \label{fig:tpr_fpr_nt}
\end{figure}

\begin{figure}[!t]
    \centering
    \begin{subfigure}{0.48\textwidth}
        \centering
        \includegraphics[width=\linewidth]{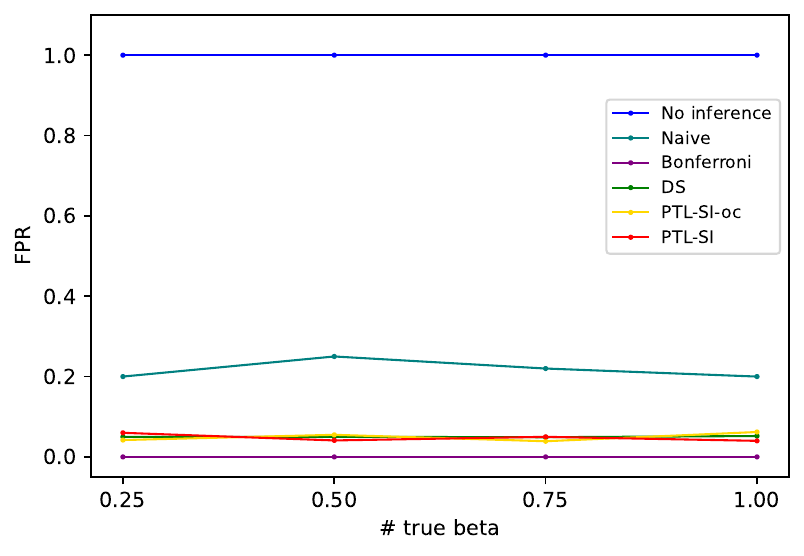}
        \caption{FPR}
        \label{fig:fpr_beta}
    \end{subfigure}\hfill
    \begin{subfigure}{0.48\textwidth}
        \centering
        \includegraphics[width=\linewidth]{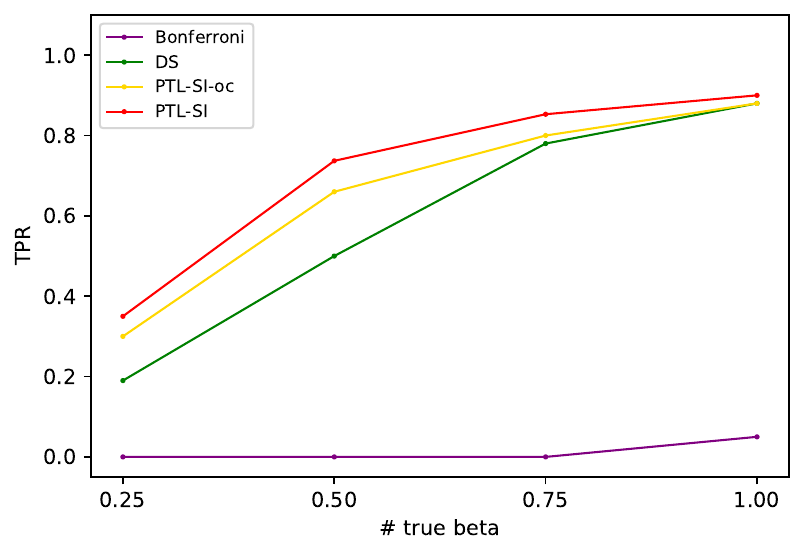}
        \caption{TPR}
        \label{fig:tpr_beta}
    \end{subfigure}
    \caption{FPR and TPR w.r.t. the true beta $\Gamma$}
    \label{fig:tpr_fpr_beta}
\end{figure}

\begin{figure}[!t]
    \centering
    \begin{subfigure}{0.48\textwidth}
        \centering
        \includegraphics[width=\linewidth]{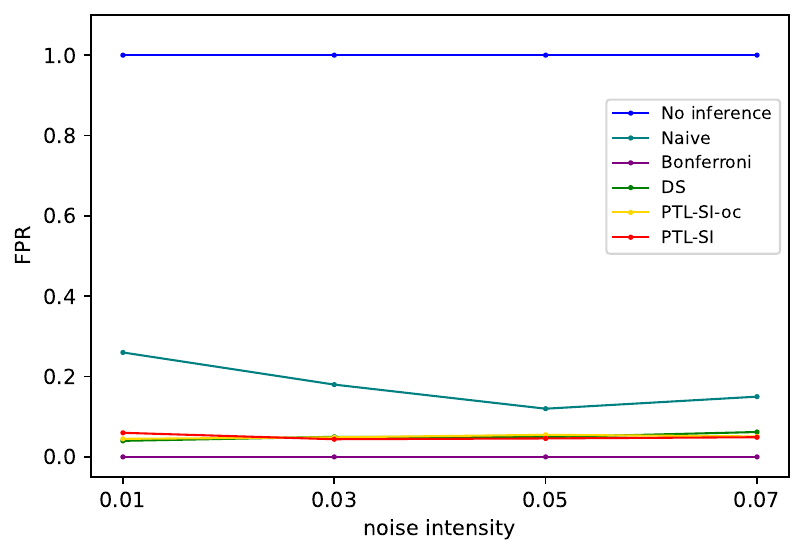}
        \caption{FPR}
        \label{fig:fpr_noise}
    \end{subfigure}\hfill
    \begin{subfigure}{0.48\textwidth}
        \centering
        \includegraphics[width=\linewidth]{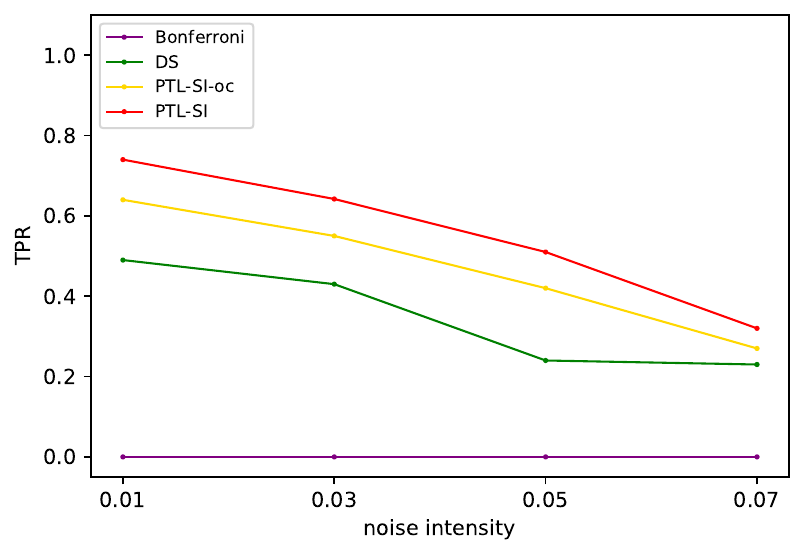}
        \caption{TPR}
        \label{fig:tpr_noise}
    \end{subfigure}
    \caption{FPR and TPR w.r.t. the noise intensity $\Upsilon$}
    \label{fig:tpr_fpr_noise}
\end{figure}

\begin{figure}[!t]
    \centering
    \begin{subfigure}{0.48\textwidth}
        \centering
        \includegraphics[width=\linewidth]{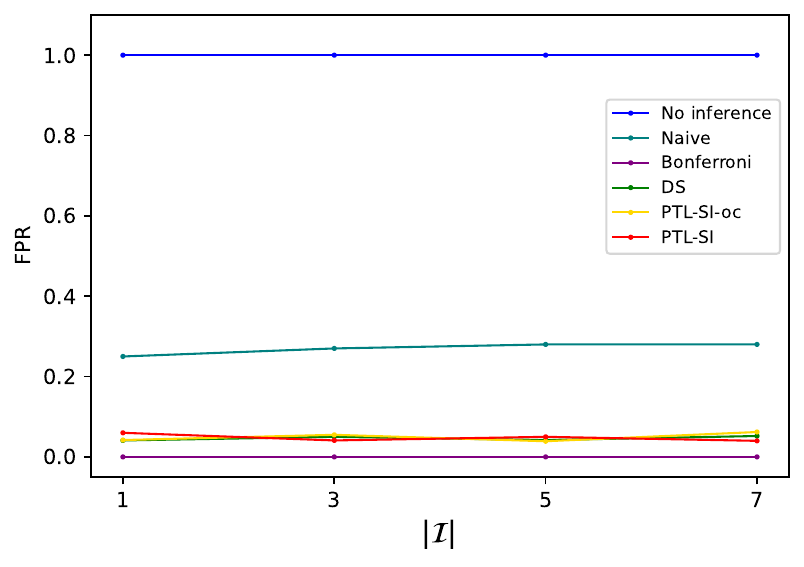}
        \caption{FPR}
        \label{fig:fpr_informative}
    \end{subfigure}\hfill
    \begin{subfigure}{0.48\textwidth}
        \centering
        \includegraphics[width=\linewidth]{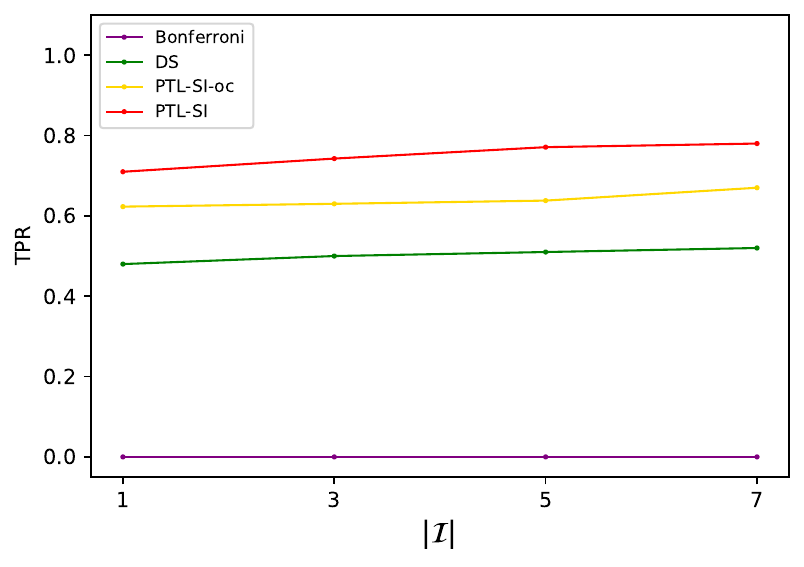}
        \caption{TPR}
        \label{fig:tpr_informative}
    \end{subfigure}
    \caption{FPR and TPR w.r.t. the number of informative auxiliary $|\cI|$}
    \label{fig:tpr_fpr_informative}
\end{figure}

\begin{figure}[!t]
    \centering
    \begin{subfigure}{0.48\textwidth}
        \centering
        \includegraphics[width=\linewidth]{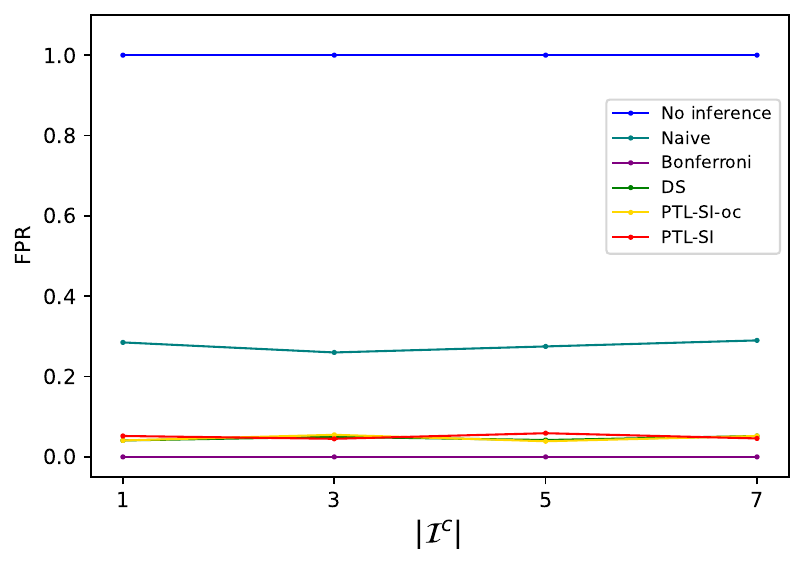}
        \caption{FPR}
        \label{fig:fpr_uninformative}
    \end{subfigure}\hfill
    \begin{subfigure}{0.48\textwidth}
        \centering
        \includegraphics[width=\linewidth]{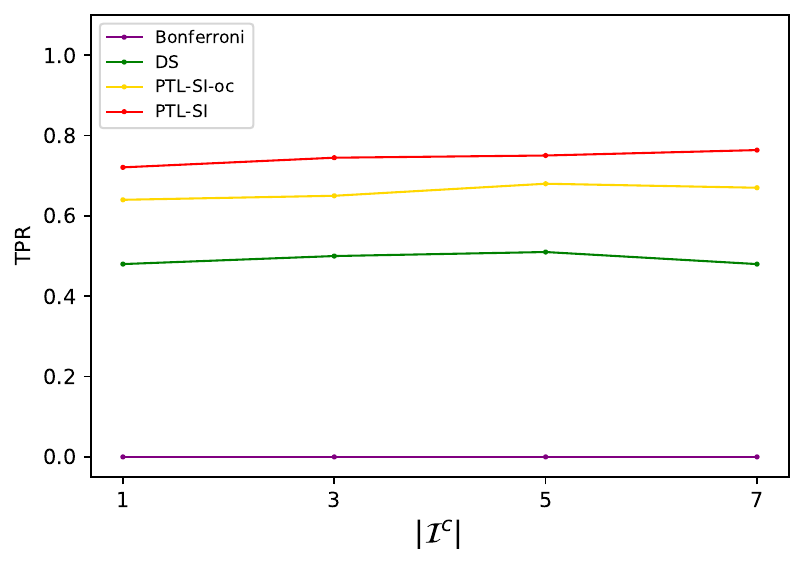}
        \caption{TPR}
        \label{fig:tpr_uninformative}
    \end{subfigure}
    \caption{FPR and TPR w.r.t. the number of uninformative auxiliary $|\cI^c|$}
    \label{fig:tpr_fpr_uninformative}
\end{figure}

\begin{figure}[!t]
    \centering
    \begin{subfigure}{0.48\textwidth}
        \centering
        \includegraphics[width=\linewidth]{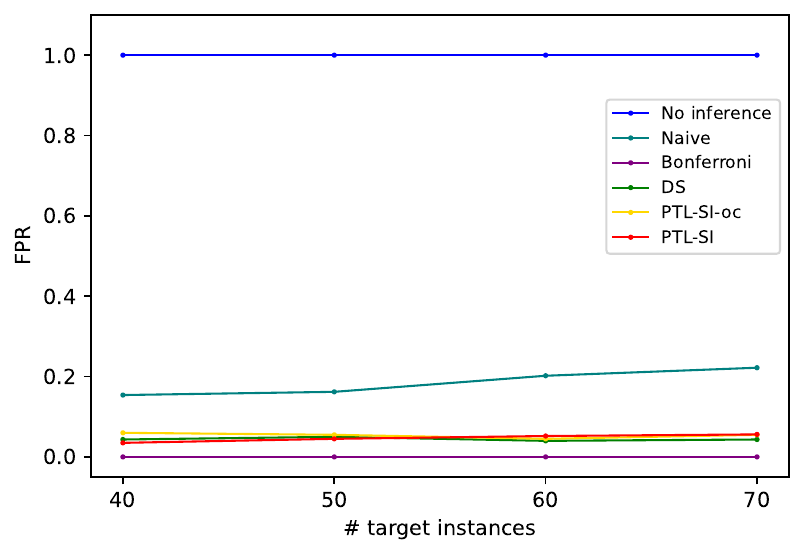}
        \caption{FPR}
        \label{fig:fpr_nt_otl}
    \end{subfigure}\hfill
    \begin{subfigure}{0.48\textwidth}
        \centering
        \includegraphics[width=\linewidth]{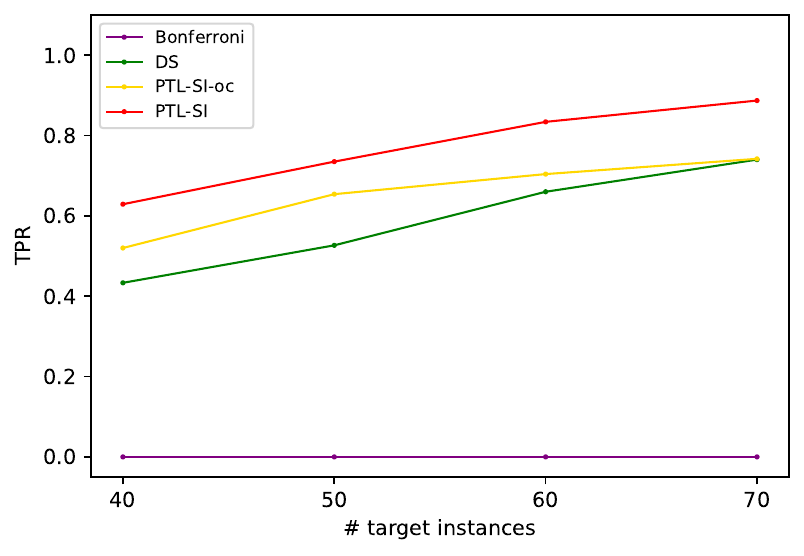}
        \caption{TPR}
        \label{fig:tpr_nt_otl}
    \end{subfigure}
    \caption{FPR and TPR in the case of Oracle Trans-Lasso w.r.t. $n_T$}
    \label{fig:tpr_fpr_nt_otl}
\end{figure}

\textbf{Noise distributions.} We considered noise following the Laplace distribution, skewnorm distribution (skewness coefficient 10), and $t_{20}$ distribution. We set $n_T \in \{40, 50, 60, 70\}$. The FPR results are shown in Fig.~\ref{fig:noise_distribution}. We confirmed that $\texttt{PTL-SI}$ still maintained good performance in FPR control. \\

\begin{figure}[!t]
    \centering
    \begin{subfigure}{0.32\textwidth}
        \centering
        \includegraphics[width=\linewidth]{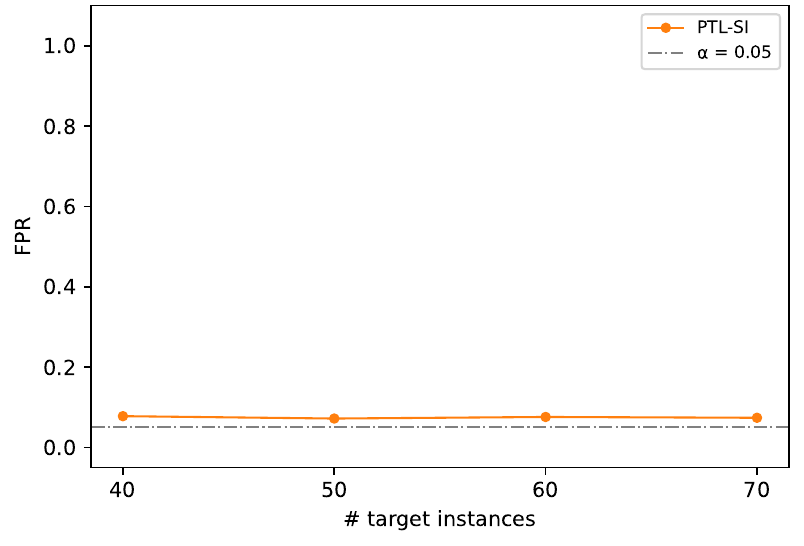}
        \caption{Laplace distribution}
        \label{fig:laplace_fpr}
    \end{subfigure}
    \begin{subfigure}{0.32\textwidth}
        \centering
        \includegraphics[width=\linewidth]{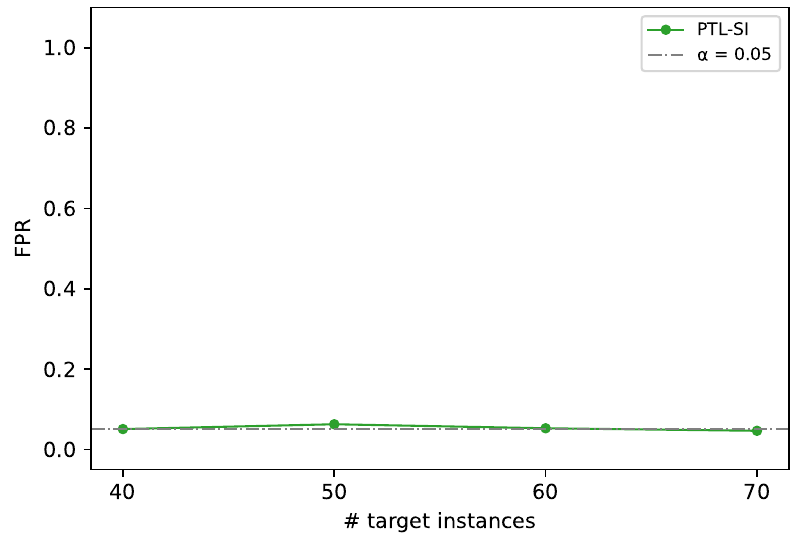}
        \caption{Skewnorm distribution}
        \label{fig:skewnorm_fpr}
    \end{subfigure}
    %
    \begin{subfigure}{0.32\textwidth}
        \centering
        \includegraphics[width=\linewidth]{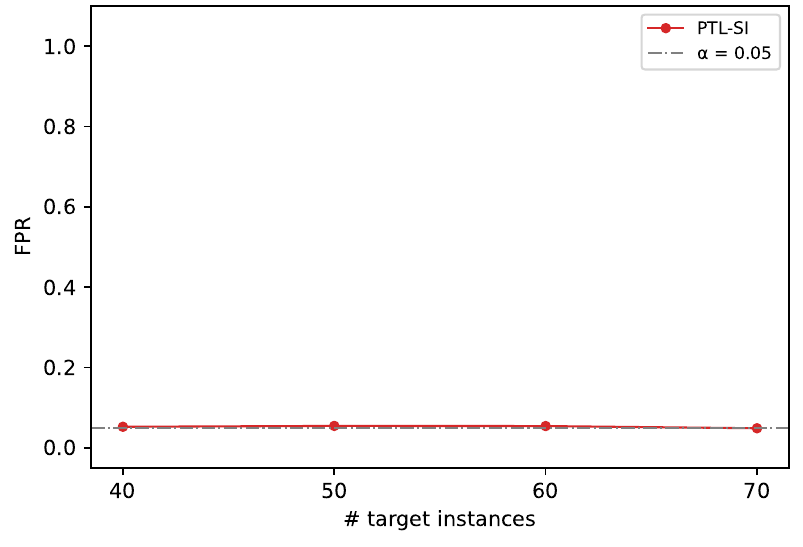}
        \caption{$t_{20}$ distribution}
        \label{fig:t20_fpr}
    \end{subfigure}
    \caption{FPR of PTL-SI under different noise distributions w.r.t. $n_T$}
    \label{fig:noise_distribution}
\end{figure}

\textbf{Computational cost.} In Fig.~\ref{fig:compcost}, we show the box-plots of the time for computing each $p$-value as well as the actual number of intervals of $z$ that we encountered. This demonstrates the reasonable computational cost of our method that scales linearly w.r.t. the complexity of the problem. \\

\begin{figure}[htbp]
    \centering
    \begin{subfigure}{0.32\textwidth}
        \centering
        \includegraphics[width=\linewidth]{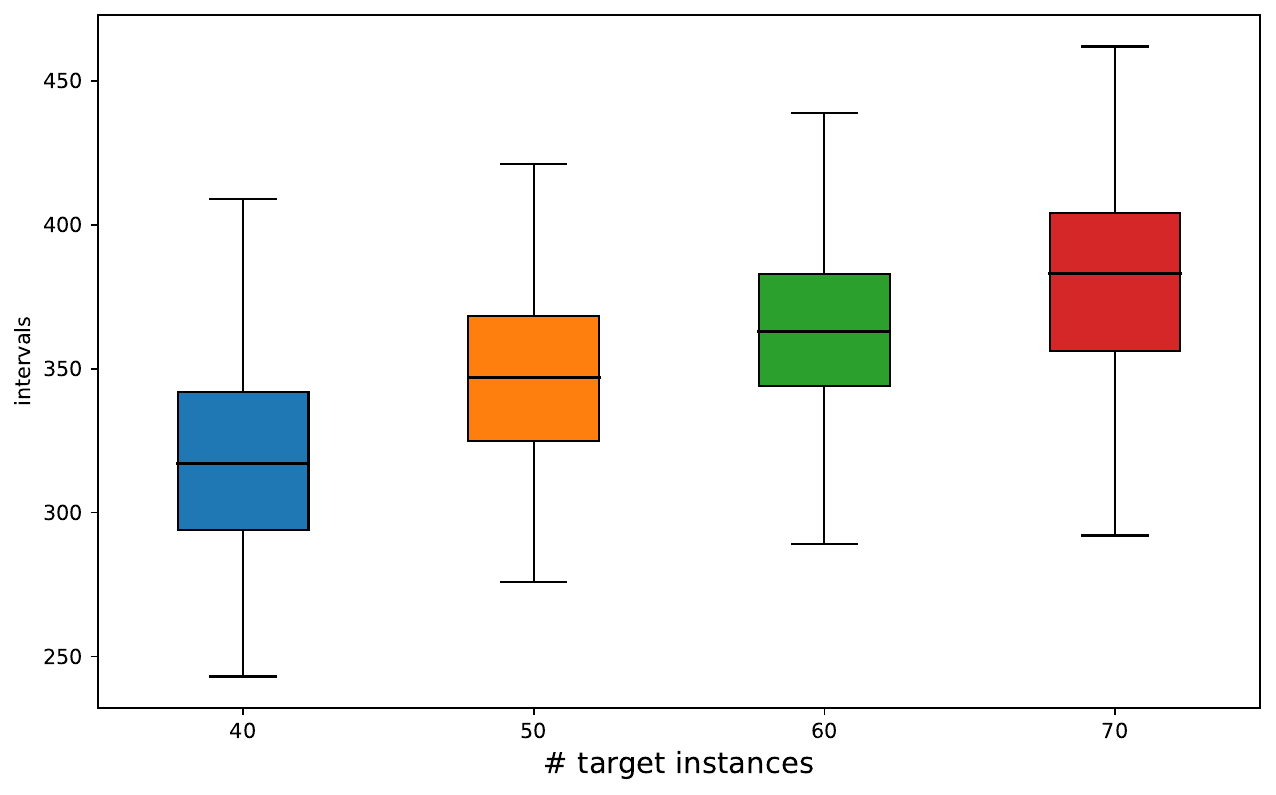}
        \caption{Identified intervals w.r.t. $n_T$}
        \label{fig:identified_intervals}
    \end{subfigure}\hfill
    \begin{subfigure}{0.32\textwidth}
        \centering
        \includegraphics[width=\linewidth]{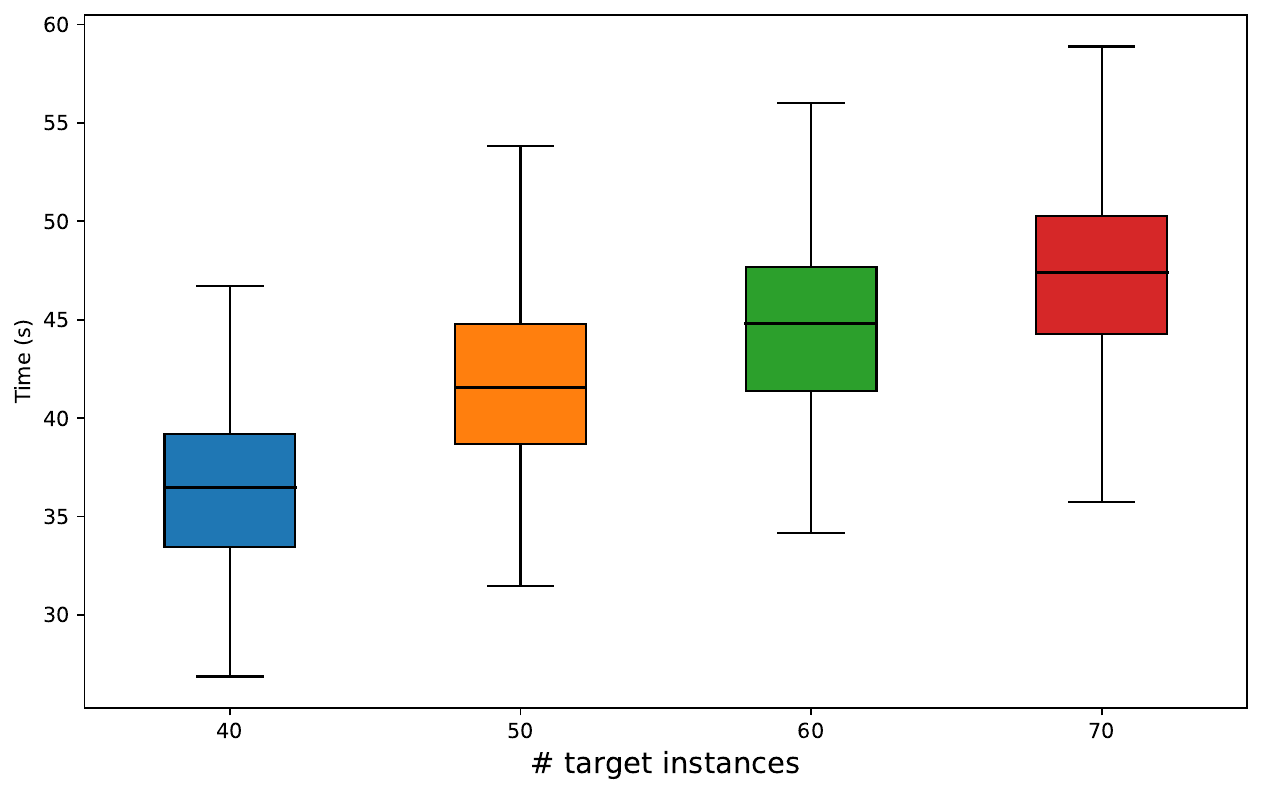}
        \caption{Computational time w.r.t. $n_T$}
        \label{fig:time_nt}
    \end{subfigure}
%
    \begin{subfigure}{0.32\textwidth}
        \centering
        \includegraphics[width=\linewidth]{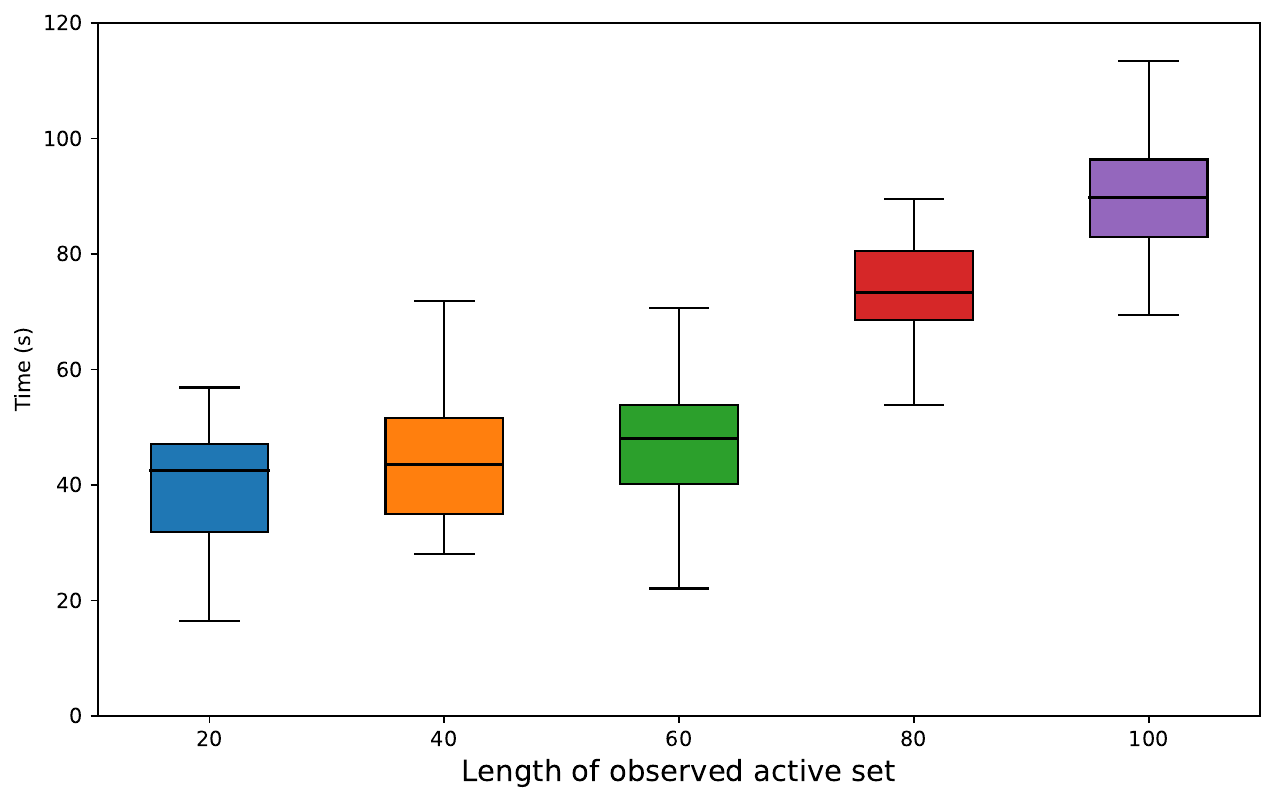}
        \caption{Computational time w.r.t. length of $\cM_{\rm obs}$}
        \label{fig:time_Mobs}
    \end{subfigure}
    \caption{Computational cost of PTL-SI}
    \label{fig:compcost}
\end{figure}

\subsection{Results on Real-World Datasets}\label{subsec5.3}

We conducted experiments on three real-world datasets: the CT-slice localization dataset, the BlogFeedback dataset and the Communities and Crime dataset (focusing on ViolentCrimesPerPop attribute), all available in the UCI Machine Learning Repository. For each dataset, we visualize the distribution of $p$-values corresponding to individual features. We used the TransFusion. Detailed experimental results are presented in Figs.~\ref{fig:ct_scan_dataset}, \ref{fig:crime_dataset} and \ref{fig:blog_dataset}. Although in certain cases the $p$-values obtained from $\texttt{PTL-SI-oc}$ are smaller than those from $\texttt{PTL-SI}$, overall, $\texttt{PTL-SI}$ consistently yields smaller p-values compared to other competing methods. This indicates that $\texttt{PTL-SI}$ demonstrates superior statistical power in detecting meaningful signals.

\begin{figure}[htbp]
    \centering
    \includegraphics[width=1\linewidth]{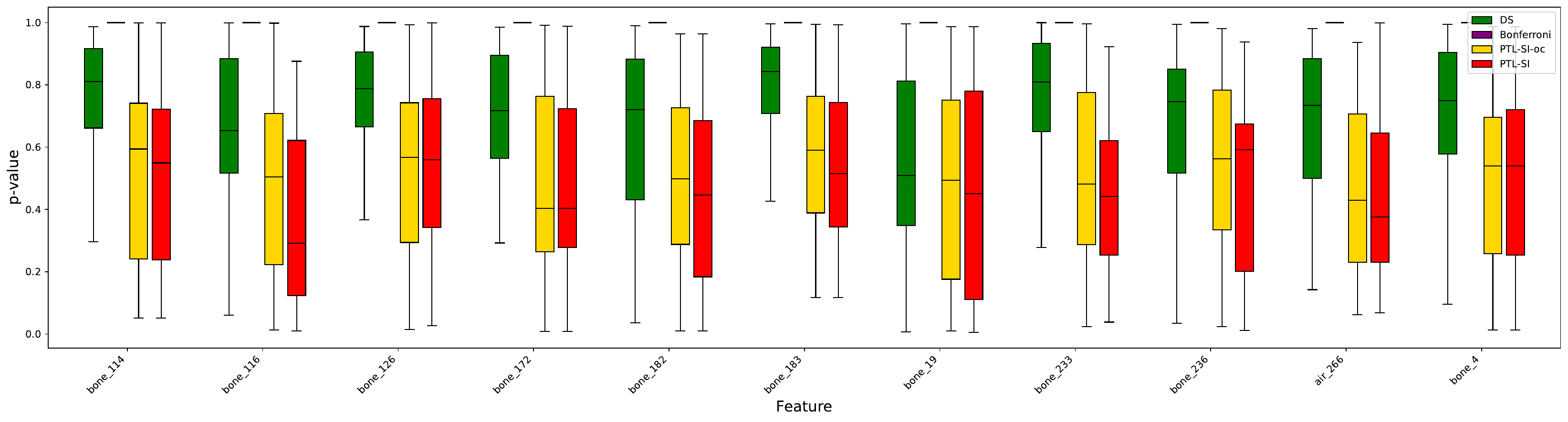}
    \caption{The CT-slice localization dataset. The target domain consists of rare cases with over 1,500 slices (50 samples). Four source domains are constructed based on slice-count intervals: fewer than 250 slices, 250–500 slices, 500–1,000 slices, and 1,000–1,500 slices, with 100 samples drawn from each source.}
    \label{fig:ct_scan_dataset}
\end{figure}

\begin{figure}[htbp]
    \centering
        \includegraphics[width=1\linewidth]{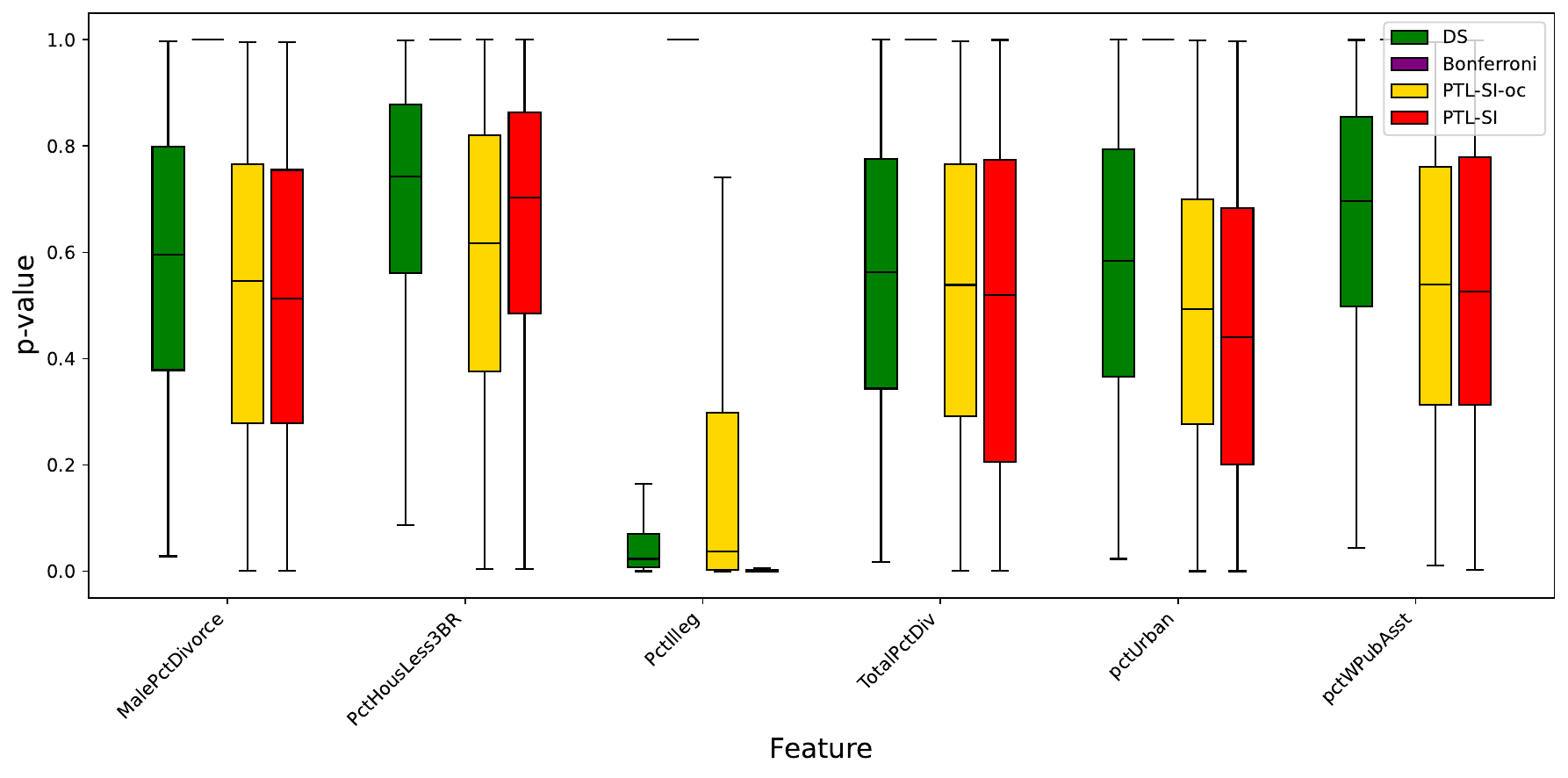}
        \caption{The Communities and Crime dataset. The target domain is defined as the state of Florida (50 samples), while six other states - New Jersey, Pennsylvania, California, Massachusetts, Ohio, and Texas - serve as source domains, each contributing 100 random samples.}
    \label{fig:crime_dataset}
\end{figure}

\begin{figure}[htbp]
    \centering
        \includegraphics[width=1\linewidth]{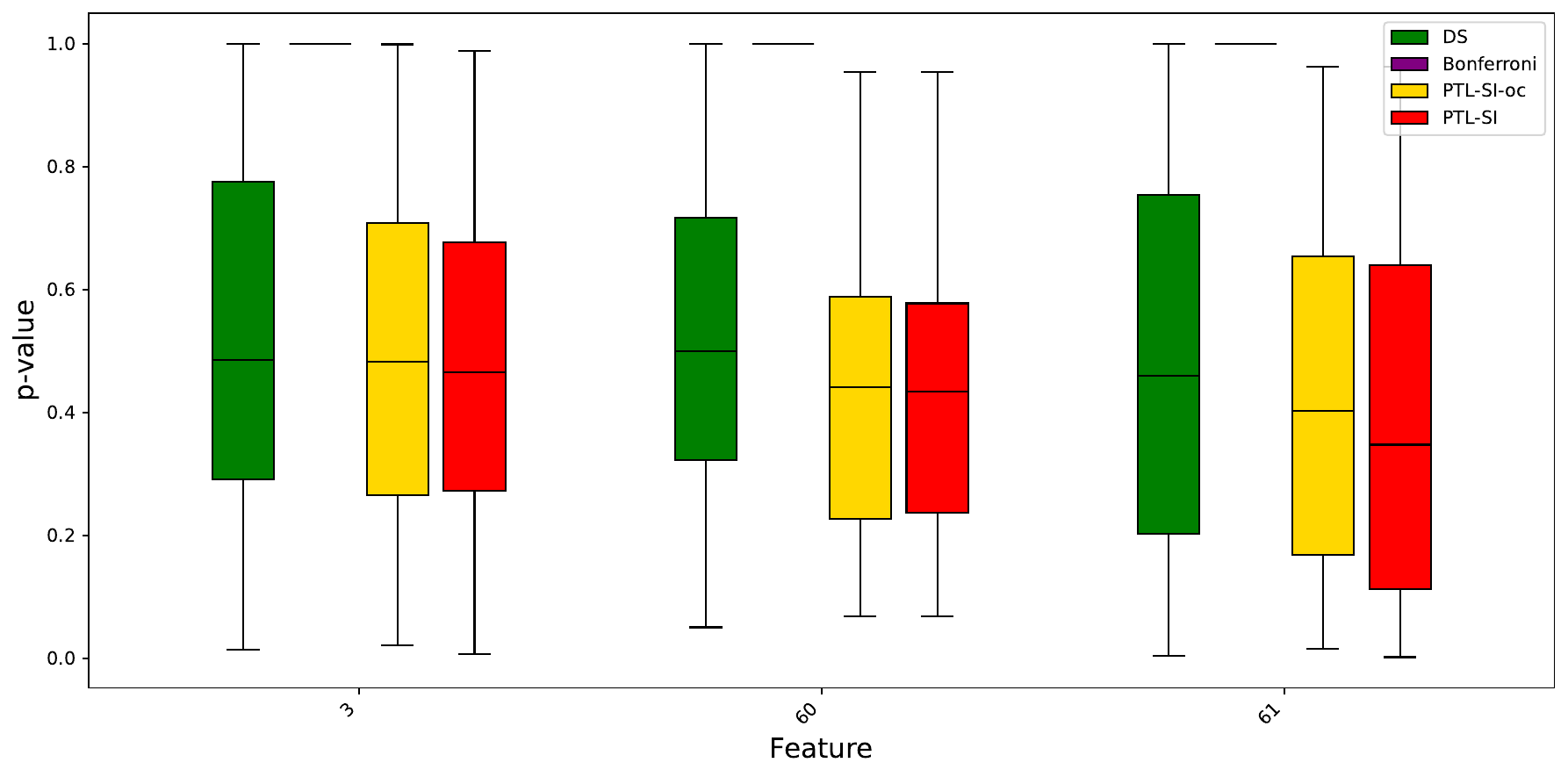}
        \caption{The BlogFeedback dataset. Domains are defined based on the median blog post length (61st column): the source domain consists of posts below the median, and the target domain includes those above. We randomly sample 300 observations from the source domain and 50 from the target domain.}
    \label{fig:blog_dataset}
\end{figure}

\section{Conclusion}\label{sec6}
In this paper, we tackled the challenge of performing valid statistical inference in high-dimensional regression (HDR) under transfer learning (TL), where traditional methods often fail due to data-dependent feature selection. We proposed PTL-SI, a novel Selective Inference (SI)-based method that provides theoretically valid control over the false positive rate (FPR) while maximizing true positives, specifically designed for the TransFusion algorithm and extended to Oracle Trans-Lasso. Our approach addresses selection bias through conditional inference and employs an efficient divide-and-conquer strategy to identify truncation regions. Future work could explore extensions to other TL frameworks and scalability improvements for ultra-high-dimensional settings. This work establishes a rigorous foundation for reliable statistical inference in TL-HDR, enhancing interpretability and trust in transfer learning applications.




\bibliographystyle{sn-mathphys}
\bibliography{ref}

\begin{appendices}
\newpage
\section{}\label{proof}
\subsection{Proof of Lemma \ref{lemma_validity_of_p_selective}}\label{proof:lemma_validity_of_p_selective}
We have
\[
\boldsymbol{\eta}_j^\top \bm{Y} \mid \left   \{ \mathcal{M}(\bm{Y}) = \mathcal{M}_{\text{obs}}, \, \cQ(\bm{Y}) = \cQ_{\text{obs}} \right\}
\
\sim \text{TN} \left( \boldsymbol{\eta}_j^\top \boldsymbol{\mu}, 
\boldsymbol{\eta}_j^\top \boldsymbol{}\Sigma \boldsymbol{\eta}_j, \mathcal{Z} \right),
\]
which is a truncated normal distribution with a mean $\boldsymbol{\eta}_j^\top \boldsymbol{\mu}$, variance $\boldsymbol{\eta}_j^\top \boldsymbol{}\Sigma \boldsymbol{\eta}_j$, in which $\boldsymbol{\mu} = \left(\begin{array}{c}
\boldsymbol{\mu}^{(1)} \\
\boldsymbol{\mu}^{(2)} \\
\vdots \\
\boldsymbol{\mu}^{(K)} \\
\boldsymbol{\mu}^{(0)}
\end{array}\right)$ and $\Sigma =
\begin{pmatrix}
\Sigma^{(1)} & 0 & \cdots & 0 & 0\\
0 & \Sigma^{(2)} & \cdots & 0 & 0 \\
\vdots & \vdots & \ddots & \vdots & \vdots \\
0 & 0 & \cdots & \Sigma^{(K)} & 0 \\
0 & 0 & \cdots & 0 & \Sigma^{(0)}  
\end{pmatrix}$, and the truncation region $\mathcal{Z}$ described in \S \ref{subsec3.2}. Therefore, under the null hypothesis,
\[
p_j^{\text{selective}} \mid \left   \{ \mathcal{M}(\bm{Y}) = \mathcal{M}_{\text{obs}}, \, \cQ(\bm{Y}) = \cQ_{\text{obs}} \right\} \sim \text{Unif}(0, 1).
\]

Thus,
\[
\mathbb{P}_{\text{H}_{0,j}}\left(p_j^{\text{selective}} \leq \alpha \mid \mathcal{M}(\bm{Y}) = \mathcal{M}_{\text{obs}}, \, \cQ(\bm{Y}) = \cQ_{\text{obs}} \right) = \alpha, \quad \forall \alpha \in [0, 1].
\]

Next, we have
\begin{align*}
\mathbb{P}_{\text{H}_{0,j}}\left(p_j^{\text{selective}} \leq \alpha \mid \mathcal{M}(\bm{Y}) = \mathcal{M}_{\text{obs}} \right)
&= \int \mathbb{P}_{\text{H}_{0,j}} \left(p_j^{\text{selective}} \leq \alpha \mid \mathcal{M}(\bm{Y}) = \mathcal{M}_{\text{obs}}, \cQ(\bm{Y}) = \cQ_{\text{obs}} \right) \\
&\quad \times \mathbb{P}_{\text{H}_{0,j}} \left( \cQ(\bm{Y}) = \cQ_{\text{obs}} \mid \mathcal{M}(\bm{Y}) = \mathcal{M}_{\text{obs}} \right) d\cQ_{\text{obs}} \\
&= \int \alpha \, \mathbb{P}_{\text{H}_{0,j}} \left( \cQ(\bm{Y}) = \cQ_{\text{obs}} \mid \mathcal{M}(\bm{Y}) = \mathcal{M}_{\text{obs}} \right) d\cQ_{\text{obs}} \\
&= \alpha \int \mathbb{P}_{\text{H}_{0,j}} \left( \cQ(\bm{Y}) = \cQ_{\text{obs}} \mid \mathcal{M}(\bm{Y}) = \mathcal{M}_{\text{obs}} \right) d\cQ_{\text{obs}} \\
&= \alpha.
\end{align*}

Finally, we obtain the result in Lemma~\ref{lemma_validity_of_p_selective} as follows:
\begin{align*}
\mathbb{P}_{\text{H}_{0,j}} \left( p_j^{\text{selective}} \leq \alpha \right)
&= \sum_{\mathcal{M}_{\text{obs}}} \mathbb{P}_{\text{H}_{0,j}} \left( p_j^{\text{selective}} \leq \alpha \mid \mathcal{M}(\bm{Y}) = \mathcal{M}_{\text{obs}} \right) \mathbb{P}_{\text{H}_{0,j}} \left( \mathcal{M}(\bm{Y}) = \mathcal{M}_{\text{obs}} \right) \\
&= \sum_{\mathcal{M}_{\text{obs}}} \alpha \, \mathbb{P}_{\text{H}_{0,j}} \left( \mathcal{M}(\bm{Y}) = \mathcal{M}_{\text{obs}} \right) \\
&= \alpha \sum_{\mathcal{M}_{\text{obs}}} \mathbb{P}_{\text{H}_{0,j}} \left( \mathcal{M}(\bm{Y}) = \mathcal{M}_{\text{obs}} \right) \\
&= \alpha.
\end{align*}

\subsection{Proof of Lemma \ref{lemma_line}}\label{proof:lemma_line}
 According to the second condition in \eqref{eq:Y}, we have:
\begin{align*}
\cQ(\bm{Y}) &= \cQ_{\text{obs}} \\
\Leftrightarrow \left( I_N - \bm{b} \boldsymbol{\eta}_j^\top \right) \bm{Y} &= \cQ_{\text{obs}}\\
\Leftrightarrow 
\bm{Y} &= \cQ_{\text{obs}} + \bm{b} \boldsymbol{\eta}_j^\top \bm{Y}
\end{align*}
By defining $\bm{a} = \cQ_{\text{obs}}$, $z = \boldsymbol{\eta}_j^\top \bm{Y}$, and incorporating the second condition of \eqref{eq:Y}, we obtain Lemma~\ref{lemma_line}.

\subsection{Proof of Lemma \ref{lem:Zu}}\label{proof:lemZu}
The set $\mathcal{Z}_u$ in \eqref{eq:cZu} can be reformulated as follows:
\[
\mathcal{Z}_u = \left\{ z \in \mathbb{R} \,\middle|\, 
\begin{array}{l}
\hat{\theta}_{j'}(z) \neq 0, \quad \forall {j'} \in \cO_u, \\
\hat{\theta}_{j'}(z) = 0, \quad \forall {j'} \notin \cO_u, \\
\text{sign} \big( \hat{\boldsymbol\theta}_{\cO_u}(z) \big) = \cS_{\cO_u}
\end{array}
\right\}
\]
The identification of $\mathcal{Z}_{u}$ is constructed based on the results presented in \cite{lee2016exact}, in which the authors characterized conditioning event of Lasso by deriving from the KKT conditions. Let us define the KKT conditions of the Weighted Lasso \eqref{eq:theta} in the TransFusion algorithm as following:
\begin{equation}\label{eq:KKT-theta}
\begin{split}
    &\frac{1}{N}  X^\top( X\hat{\boldsymbol{\theta}}(z) - \bm{Y}(z)) + \lambda_0 \tilde{\bm{a}} \circ \cS(z) = \bm 0, \\
    &\cS_{j'} = \text{sign}(\hat{\theta}_{j'}(z)), \quad  \text{if } \hat{\theta}_{j'}(z) \neq 0,\\
&\cS_{j'} \in (-1,1), \quad \text{if } \hat{\theta}_{j'}(z) = 0,
\end{split}
\end{equation}
where the operator $\circ$ is element-wise product, $\tilde{\bm{a}} = \left(a_1 \bm 1_p^\top, \dots, a_K \bm 1_p^\top, a_0 \bm 1_p^\top\right)^
\top$ with $a_0 = 1$ and $\bm 1_p \in \mathbb{R}^p$ is the all-ones vector.\\

By partitioning Equation~\eqref{eq:KKT-theta} with respect to the active set $\cO_{u}$, where $\cO_{u}^{c}$ denotes its complement, we reformulate the KKT conditions as:
\begin{equation}\label{eq:KKT-theta2}
    \begin{split}
        &\frac{1}{N}  X^\top_{\cO_u}( X_{\cO_u}\hat{\boldsymbol{\theta}}_{\cO_u}(z) - \bm{Y}(z)) + \lambda_0 \tilde{\bm{a}}_{\cO_u} \circ \cS_{\cO_u} = \bm 0, \\
        &\frac{1}{N}  X^\top_{\cO_u^c}( X_{\cO_u}\hat{\boldsymbol{\theta}}_{\cO_u}(z) - \bm{Y}(z)) + \lambda_0 \tilde{\bm{a}}_{\cO_u^c} \circ \cS_{\cO_u^c} = \bm 0, \\
        &\text{sign} \big( \hat{\boldsymbol\theta}_{\cO_u}(z) \big) = \cS_{\cO_u},\\ 
        &\| \cS_{\cO_{u}^c} \|_{\infty} < \bm 1.
    \end{split}
\end{equation}
Solving the first two equations in \eqref{eq:KKT-theta2} for $\hat{\boldsymbol{\theta}}_{\cO_u}(z)$ and $\cS_{\cO_u^c}$ yields the equivalent conditions:
\begin{align}
    &\hat{\boldsymbol{\theta}}_{\cO_u}(z) = \left( X^\top_{\cO_u} X_{\cO_u}\right)^{-1} \left( X^\top_{\cO_u}\bm{Y}(z)-N\lambda_0 \tilde{\bm{a}}_{\cO_u} \circ \cS_{\cO_u}\right), \label{eq:theta_u} \\
    &\cS_{\cO_u^c} 
    =  X^\top_{\cO_u^c}\left( X_{\cO_u}^\top \right)^{+}\tilde{\bm{a}}_{\cO_u} \circ \cS_{\cO_u}\oslash \tilde{\bm{a}}_{\cO_u^c} 
 + \frac{1}{\lambda_0 N} X^\top_{\cO_u^c}\left(I_N -  X_{\cO_u}\left( X_{\cO_u}\right)^{+}\right)\bm{Y}(z)\oslash \tilde{\bm{a}}_{\cO_u^c}, \nonumber \\
    &\operatorname{sign} \big( \hat{\boldsymbol\theta}_{\cO_u}(z) \big) = \cS_{\cO_u}, \nonumber \\ 
    &\| \cS_{\cO_{u}^c} \|_{\infty} < \bm 1, \nonumber
\end{align}
where the operator $\oslash$ is element-wise division, $( X)^+ = ( X^\top  X)^{-1}  X^\top$, and $( X^\top)^+ =  X ( X^\top  X)^{-1}$. Then, the set $\mathcal{Z}_u$ can be rewritten as:
\[
\resizebox{\linewidth}{!}{$
\mathcal{Z}_u = \left\{ z \in \mathbb{R} \,\middle|\, 
\begin{aligned}
& \hat{\boldsymbol{\theta}}_{\cO_u}(z) = \left( X^\top_{\cO_u} X_{\cO_u}\right)^{-1} 
   \left( X^\top_{\cO_u}\bm{Y}(z)-N\lambda_0 \tilde{\bm{a}}_{\cO_u} \circ \cS_{\cO_u}\right), \\
& \cS_{\cO_u^c} =  X^\top_{\cO_u^c}\left( X_{\cO_u}^\top \right)^{+}\tilde{\bm{a}}_{\cO_u} \circ \cS_{\cO_u} \oslash \tilde{\bm{a}}_{\cO_u^c}  + \frac{1}{\lambda_0 N} X^\top_{\cO_u^c}\left(I_N -  X_{\cO_u}\left( X_{\cO_u}\right)^{+}\right)\bm{Y}(z) \oslash \tilde{\bm{a}}_{\cO_u^c}, \\
& \operatorname{sign} \big( \hat{\boldsymbol\theta}_{\cO_u}(z) \big) = \cS_{\cO_u}, \\ 
& \| \cS_{\cO_{u}^c} \|_{\infty} < \bm 1.
\end{aligned}
\right\}
$}
\]

The two last conditions of $\mathcal{Z}_u$ then can be rewritten as:
\begin{align*}
&\quad \:\left\{ \text{sign} \big( \hat{\boldsymbol\theta}_{\cO_u}(z) \big) = \cS_{\cO_u} \right\} \\
&= \left\{ \cS_{\cO_u} \circ \hat{\boldsymbol\theta}_{\cO_u}(z) > \bm 0 \right\}, \\
&= \left\{ \cS_{\cO_u} \circ \left( X^\top_{\cO_u} X_{\cO_u}\right)^{-1} \left ( X^\top_{\cO_u}\bm{Y}(z)-N\lambda_0 \tilde{\bm{a}}_{\cO_u} \circ \cS_{\cO_u} \right) > \bm 0 \right\}, \\
&= \left\{ 
    \cS_{\cO_u} \circ \left( X_{\cO_u}\right)^{+}\bm{Y}(z) > N\lambda_0 \cS_{\cO_u} \circ \left(\left( X^\top_{\cO_u} X_{\cO_u}\right)^{-1}\tilde{\bm{a}}_{\cO_u} \circ \cS_{\cO_u} \right)
\right\}, \\
&= \left\{  \cS_{\cO_u} \circ \left( X_{\cO_u}\right)^{+}(\bm{a} + \bm{b}z) > N\lambda_0 \cS_{\cO_u} \circ \left( \left( X^\top_{\cO_u} X_{\cO_u}\right)^{-1}\tilde{\bm{a}}_{\cO_u} \circ \cS_{\cO_u} \right)
 \right\}, \\
&= \left\{ \boldsymbol{\psi}_0 z \leq \boldsymbol{\gamma}_0 \right\}
\end{align*}
where 
\[
\boldsymbol{\psi}_0 = -\cS_{\cO_u} \circ \left( X_{\cO_u}\right)^{+} \bm{b} \, ,
\]
\[
\boldsymbol{\gamma}_0 = \cS_{\cO_u} \circ \left( X_{\cO_u}\right)^{+} \bm{a} - N\lambda_0 \cS_{\cO_u} \circ \left( \left( X^\top_{\cO_u} X_{\cO_u}\right)^{-1}\tilde{\bm{a}}_{\cO_u} \circ \cS_{\cO_u} \right).
\]
\resizebox{\linewidth}{!}{$
\begin{aligned}
    & \quad \:\left\{ \| \cS_{\cO_{u}^c} \|_{\infty} < \bm 1 \right\} \\
    &= \left\{ -\bm 1 < \cS_{\cO_{u}^c} < \bm 1 \right\}, \\
    &= \left\{ -\bm 1 <  X^\top_{\cO_u^c}\left( X_{\cO_u}^\top \right)^{+}\tilde{\bm{a}}_{\cO_u} \circ \cS_{\cO_u}\oslash \tilde{\bm{a}}_{\cO_u^c} + \frac{1}{\lambda_0 N} X^\top_{\cO_u^c}\left (I_N -  X_{\cO_u}\left( X_{\cO_u}\right)^{+}\right)\oslash \tilde{\bm{a}}_{\cO_u^c}\bm{Y}(z) < \bm 1 \right\}, \\
    &= \left\{
    \begin{aligned}       
     \frac{1}{\lambda_0 N} X^\top_{\cO_u^c}\left (I_N -  X_{\cO_u}\left( X_{\cO_u}\right)^{+}\right)\oslash \tilde{\bm{a}}_{\cO_u^c}\bm{Y}(z) &< \bm 1 -  X^\top_{\cO_u^c}\left( X_{\cO_u}^\top \right)^{+}\tilde{\bm{a}}_{\cO_u} \circ \cS_{\cO_u}\oslash \tilde{\bm{a}}_{\cO_u^c}  , \\
    - \frac{1}{\lambda_0 N} X^\top_{\cO_u^c}\left (I_N -  X_{\cO_u}\left( X_{\cO_u}\right)^{+}\right)\oslash \tilde{\bm{a}}_{\cO_u^c}\bm{Y}(z) &< \bm 1 +  X^\top_{\cO_u^c}\left( X_{\cO_u}^\top \right)^{+}\tilde{\bm{a}}_{\cO_u} \circ \cS_{\cO_u}\oslash \tilde{\bm{a}}_{\cO_u^c} 
       \end{aligned}
\right\}, \\
  &= \left\{
    \begin{aligned}       
     \frac{1}{\lambda_0 N} X^\top_{\cO_u^c}\left (I_N -  X_{\cO_u}\left( X_{\cO_u}\right)^{+}\right)\oslash \tilde{\bm{a}}_{\cO_u^c}(\bm{a} + \bm{b}z) &< \bm 1 -  X^\top_{\cO_u^c}\left( X_{\cO_u}^\top \right)^{+}\tilde{\bm{a}}_{\cO_u} \circ \cS_{\cO_u}\oslash \tilde{\bm{a}}_{\cO_u^c}  , \\
    - \frac{1}{\lambda_0 N} X^\top_{\cO_u^c}\left (I_N -  X_{\cO_u}\left( X_{\cO_u}\right)^{+}\right)\oslash \tilde{\bm{a}}_{\cO_u^c}(\bm{a} + \bm{b}z) &< \bm 1 +  X^\top_{\cO_u^c}\left( X_{\cO_u}^\top \right)^{+}\tilde{\bm{a}}_{\cO_u} \circ \cS_{\cO_u}\oslash \tilde{\bm{a}}_{\cO_u^c} 
       \end{aligned}
\right\}, \\
    &= \left\{ \boldsymbol{\psi}_1 z \leq \boldsymbol{\gamma}_1 \right\}
\end{aligned}
$}\\

where
\[
\boldsymbol{\psi}_1 =  \begin{pmatrix}    \frac{1}{\lambda_0 N} X^\top_{\cO_u^c}\left (I_N -  X_{\cO_u}\left( X_{\cO_u}\right)^{+}\right)\oslash \tilde{\bm{a}}_{\cO_u^c}\bm{b} \\
    -\frac{1}{\lambda_0 N} X^\top_{\cO_u^c}\left (I_N -  X_{\cO_u}\left( X_{\cO_u}\right)^{+}\right)\oslash \tilde{\bm{a}}_{\cO_u^c}\bm{b}
    \end{pmatrix},
\]
\[
\resizebox{1\linewidth}{!}{$
\boldsymbol{\gamma}_1 = \begin{pmatrix}
\bm 1 -  X^\top_{\cO_u^c}\left( X_{\cO_u}^\top \right)^{+}\tilde{\bm{a}}_{\cO_u} \circ \cS_{\cO_u}\oslash \tilde{\bm{a}}_{\cO_u^c} - \frac{1}{\lambda_0 N} X^\top_{\cO_u^c}\left (I_N -  X_{\cO_u}\left( X_{\cO_u}\right)^{+}\right)\oslash \tilde{\bm{a}}_{\cO_u^c}\bm{a} \\
\bm 1 +  X^\top_{\cO_u^c}\left( X_{\cO_u}^\top \right)^{+}\tilde{\bm{a}}_{\cO_u} \circ \cS_{\cO_u}\oslash \tilde{\bm{a}}_{\cO_u^c} + \frac{1}{\lambda_0 N} X^\top_{\cO_u^c}\left (I_N -  X_{\cO_u}\left( X_{\cO_u}\right)^{+}\right)\oslash \tilde{\bm{a}}_{\cO_u^c}\bm{a}
\end{pmatrix}.
$}
\]\\

Finally, the set $\mathcal{Z}_u$ can be defined as:
\[
\mathcal{Z}_u = \{ z \in \mathbb{R} \mid \boldsymbol\psi z \leq \boldsymbol\gamma \},
\]
where $\boldsymbol\psi = (\boldsymbol\psi_0 \quad \boldsymbol\psi_1)^{\top}$, $\boldsymbol\gamma = (\boldsymbol\gamma_0 \quad \boldsymbol\gamma_1)^{\top}$.\\

Thus, the set $\mathcal{Z}_u$ can be identified by solving a system of linear inequalities with respect to $z$.

\subsection{Proof of Lemma \ref{lem:Zv}}\label{proof:lemZv}
The set $\mathcal{Z}_v$ in \eqref{eq:cZv} can be reformulated as follows:
\[
\quad \mathcal{Z}_v = \left\{ z \in \mathbb{R} \,\middle|\, 
\begin{array}{l}
\hat{\delta}_{j''}(z) \neq 0, \quad \forall {j''} \in \cL_v, \\
\hat{\delta}_{j''}(z) = 0, \quad \forall {j''} \notin \cL_v, \\
\text{sign} \big( \hat{\boldsymbol\delta}_{\cL_v}(z) \big) = \cS_{\cL_v}
\end{array}
\right\}
\]

Analogous to $\mathcal{Z}_u$, we characterize $\mathcal{Z}_v$ by deriving from the KKT conditions. The KKT conditions of the Lasso \eqref{eq:delta} in the TransFusion algorithm after obtaining $\cO_u,\,\cS_{\cO_u}$ as following:
\begin{equation}\label{eq:KKT-delta}
\begin{split}
    &\frac{1}{n_T} ( X^{(0)})^\top \left( X^{(0)} \hat{\boldsymbol{\delta}}(z) -  \left (\bm{Y}^{(0)}(z) -  X^{(0)} \hat{\boldsymbol w}_u(z)\right)\right) + \tilde{\lambda} \cS(z) = \bm 0,\\
&\cS_{j''} = \text{sign}(\hat{\delta}_{j''}(z)), \quad  \text{if } \hat{\delta}_{j''}(z) \neq 0, \\
& \cS_{j''} \in (-1,1), \quad \text{if } \hat{\delta}_{j''}(z) = 0.
\end{split}
\end{equation}

First, we compute $\hat{\boldsymbol w}_u(z)$ in Eq. \eqref{eq:w} based on $\hat{\boldsymbol{\theta}}_{\cO_u}(z)$ in $\eqref{eq:theta_u}$:
\begin{equation}\label{eq:w_u}
    \begin{split}
            \hat{\boldsymbol w}_u(z) &=  \frac{1}{N}B\hat{\boldsymbol{\theta}}_u(z) = \frac{1}{N}BE_u\hat{\boldsymbol{\theta}}_{\cO_u}(z)\\ &= \frac{1}{N}BE_u \left( X^\top_{\cO_u} X_{\cO_u}\right)^{-1} \left ( X^\top_{\cO_u}\bm{Y}(z)-N\lambda_0 \mathbf{\tilde{a}}_{\cO_u} \circ \cS_{\cO_u}\right ) 
    \end{split}
\end{equation}
where 
\[
B = \left(n_SI_p,\dots,n_SI_p,\left(Kn_S + n_T\right)I_p \right) \in\mathbb{R}^{p\times (K+1)p},
\]
\[E_u = \left( \bm{e}_{j_1}, \bm{e}_{j_2},\dots, \bm{e}_{j_{|\cO_u|}} \right) \in \mathbb{R}^{(K+1)p \times |\cO_u| }, \quad j_k \in \cO_u.\]\\

Next, we define: 
\begin{equation}\label{eq:Q}
    Q = \left(0_{n_T \times Kn_S}, I_{n_T} \right) \in \mathbb{R}^{n_T \times N}
\end{equation}
and obtain $\bm{Y}^{(0)}(z) = Q\bm{Y}(z)$. \\

The first condition in \eqref{eq:KKT-delta} is equivalent to:
\[
\frac{1}{n_T} ( X^{(0)})^\top \left( X^{(0)} \hat{\boldsymbol{\delta}}(z) -  \boldsymbol{\phi}_u\bm{Y}(z) - \boldsymbol\iota_u\right) + \tilde{\lambda} \cS(z) = \bm 0
\]
where
\[
\boldsymbol{\phi}_u = Q - \frac{1}{N} X^{(0)}BE_u \left( X^\top_{\cO_u} X_{\cO_u}\right)^{-1}  X^\top_{\cO_u},
\]
\[
\boldsymbol\iota_u =\lambda_0  X^{(0)}BE_u \left( X^\top_{\cO_u} X_{\cO_u}\right)^{-1} \bm{\tilde{a}}_{\cO_u} \circ \cS_{\cO_u}.
\]\\
By partitioning Equation~\eqref{eq:KKT-delta} with respect to the active set $\cL_{v}$, where $\cL_{v}^{c}$ denotes its complement, we reformulate the KKT conditions as:
\begin{equation}
\begin{split}
\label{eq:KKT-delta2}
&\frac{1}{n_T} ( X^{(0)}_{\cL_v})^\top \left( X^{(0)}_{\cL_v} \hat{\boldsymbol{\delta}}_{\cL_v}(z) - \boldsymbol{\phi}_u\bm{Y}(z) - \boldsymbol\iota_u\right) + \tilde{\lambda} \cS_{\cL_v} = \bm 0, \\
&\frac{1}{n_T} ( X^{(0)}_{\cL_v^c})^\top \left( X^{(0)}_{\cL_v} \hat{\boldsymbol{\delta}}_{\cL_v}(z) - \boldsymbol{\phi}_u\bm{Y}(z) - \boldsymbol\iota_u\right) + \tilde{\lambda} \cS_{\cL_v^c} = \bm 0,\\
&\text{sign} \big( \hat{\boldsymbol\delta}_{\cL_v}(z) \big) = \cS_{\cL_v}, \\
&\| \cS_{\cL_{v}^c} \|_{\infty} < \bm 1.
\end{split}
\end{equation}
Solving the first two equations in \eqref{eq:KKT-delta2} for $\hat{\boldsymbol{\delta}}_{\cL_v}(z)$ and $\cS_{\cL_v^c}$ yields the equivalent conditions:

\begin{align}
    &\hat{\boldsymbol{\delta}}_{\cL_v}(z) =   \left(\left( X^{(0)}_{\cL_v}\right)^\top X^{(0)}_{\cL_v}\right)^{-1} 
    \left( \left( X^{(0)}_{\cL_v}\right)^\top\left(\boldsymbol{\phi}_u\bm{Y}(z) + \boldsymbol\iota_u\right) 
    - n_T\tilde{\lambda} \cS_{\cL_v}\right ),\label{eq:delta_v} \\
    &\resizebox{\textwidth}{!}{$
    \cS_{\cL_v^c} = 
    \left( X^{(0)}_{\cL_v^c}\right)^\top
    \left(\left( X^{(0)}_{\cL_v}\right)^\top\right)^{+}
    \cS_{\cL_v} + \frac{1}{\tilde{\lambda}n_T}
    \left( X^{(0)}_{\cL_v^c}\right)^\top
    \left(I_{n_T} -  X^{(0)}_{\cL_v} 
    \left( X^{(0)}_{\cL_v}\right)^{+}\right)
    \left(\boldsymbol{\phi}_{u}\bm{Y}(z)+\boldsymbol\iota_u\right),$}
    \notag \\   
    &\text{sign} \big( \hat{\boldsymbol\delta}_{\cL_v}(z) \big) = \cS_{\cL_v},
    \notag \\
    &\| \cS_{\cL_{v}^c} \|_{\infty} < \bm 1.
    \notag
\end{align}
Then, the set $\mathcal{Z}_v$ can be rewritten as:

\resizebox{\linewidth}{!}{$
\mathcal{Z}_v = \left\{ z \in \mathbb{R} \,\middle|\,
\begin{aligned}
    &\hat{\boldsymbol{\delta}}_{\cL_v}(z) = \left(\left( X^{(0)}_{\cL_v}\right)^\top X^{(0)}_{\cL_v}\right)^{-1} \left( \left( X^{(0)}_{\cL_v}\right)^\top\left(\boldsymbol{\phi}_u\bm{Y}(z) + \boldsymbol\iota_u\right) - n_T\tilde{\lambda} \cS_{\cL_v}\right ), \\
    &\cS_{\cL_v^c} =\left( X^{(0)}_{\cL_v^c}\right)^\top\left(\left( X^{(0)}_{\cL_v}\right)^\top\right)^{+}\cS_{\cL_v}  + \frac{1}{\tilde{\lambda}n_T}\left( X^{(0)}_{\cL_v^c}\right)^\top\left(I_{n_T} -  X^{(0)}_{\cL_v} \left( X^{(0)}_{\cL_v}\right)^{+}\right)\left(\boldsymbol{\phi}_{u}\bm{Y}(z)+\boldsymbol\iota_u\right),  \\
    &\text{sign} \big( \hat{\boldsymbol\delta}_{\cL_v}(z) \big) = \cS_{\cL_v}, \\
    &\| \cS_{\cL_{v}^c} \|_{\infty} < \bm 1. 
\end{aligned}
\right\}
$}\\ \\

The two last conditions of $\mathcal{Z}_v$ then can be rewritten as:
\begin{align*}
& \quad \; \left\{ \text{sign} \big( \hat{\boldsymbol\delta}_{\mathcal{L}_v}(z) \big) = \mathcal{S}_{\mathcal{L}_v} \right\} \\ 
&= \left\{ \mathcal{S}_{\mathcal{L}_v} \circ \hat{\boldsymbol\delta}_{\mathcal{L}_v}(z) > \boldsymbol{0} \right\} \\ 
&= \left\{ \mathcal{S}_{\mathcal{L}_v} \circ \left(\left( X^{(0)}_{\mathcal{L}_v} \right)^\top X^{(0)}_{\mathcal{L}_v} \right)^{-1} 
\left( \left( X^{(0)}_{\mathcal{L}_v} \right)^\top \left( \boldsymbol{\phi}_u \boldsymbol{Y}(z) + \boldsymbol{\iota}_u \right) - n_T \tilde{\lambda} \mathcal{S}_{\mathcal{L}_v} \right) > \boldsymbol{0} \right\} \\ 
&= \left\{ \mathcal{S}_{\mathcal{L}_v} \circ \left( X^{(0)}_{\mathcal{L}_v} \right)^{+} 
\left( \boldsymbol{\phi}_u \boldsymbol{Y}(z) + \boldsymbol{\iota}_u \right) > 
\tilde{\lambda} n_T \mathcal{S}_{\mathcal{L}_v} \circ 
\left( \left( \left( X^{(0)}_{\mathcal{L}_v} \right)^\top X^{(0)}_{\mathcal{L}_v} \right)^{-1} \mathcal{S}_{\mathcal{L}_v} \right) \right\} \\ 
&= \left\{ \mathcal{S}_{\mathcal{L}_v} \circ \left( X^{(0)}_{\mathcal{L}_v} \right)^{+} 
\left( \boldsymbol{\phi}_u ( \boldsymbol{a} + \boldsymbol{b} z ) + \boldsymbol{\iota}_u \right) > 
\tilde{\lambda} n_T \mathcal{S}_{\mathcal{L}_v} \circ 
\left( \left( \left( X^{(0)}_{\mathcal{L}_v} \right)^\top X^{(0)}_{\mathcal{L}_v} \right)^{-1} \mathcal{S}_{\mathcal{L}_v} \right) \right\} \\ 
&= \left\{ \boldsymbol{\nu}_0 z \leq \boldsymbol{\kappa}_0 \right\}
\end{align*}

where 
\[
\boldsymbol{\nu}_0 = -\cS_{\cL_v} \circ \left( X^{(0)}_{\cL_v}\right)^{+}\boldsymbol\phi_u\bm{b}\, ,
\]
\[
\boldsymbol{\kappa}_0 = \cS_{\cL_v} \circ \left( X^{(0)}_{\cL_v}\right)^{+}(\boldsymbol\phi_u\bm{a} + \boldsymbol\iota_u)
- \tilde{\lambda} n_T \cS_{\cL_v} \circ
\left( \left(\left( X^{(0)}_{\cL_v}\right)^\top  X^{(0)}_{\cL_v}\right)^{-1} \cS_{\cL_v} \right).
\]
\resizebox{\linewidth}{!}{%
$\begin{aligned}
    & \quad \:\left\{ \| \cS_{\cL_{v}^c} \|_{\infty} < \bm 1 \right\} \\
    &= \left\{ -\bm 1 < \cS_{\cL_{v}^c} < \bm 1 \right\}, \\
    &= \left\{ -\bm 1 < ( X^{(0)}_{\cL_v^c})^\top\left(( X^{(0)}_{\cL_v})^\top\right)^{+}\cS_{\cL_v} + \frac{1}{\tilde{\lambda}n_T}( X^{(0)}_{\cL_v^c})^\top\left(I_{n_T} -  X^{(0)}_{\cL_v} \left( X^{(0)}_{\cL_v}\right)^{+}\right)\left(\boldsymbol{\phi}_{u}\bm{Y}(z)+\boldsymbol\iota_u\right) 
    < \bm 1 \right\}, \\
    &= \left\{
    \begin{aligned}       
    \frac{1}{\tilde{\lambda}n_T}( X^{(0)}_{\cL_v^c})^\top\left(I_{n_T} -  X^{(0)}_{\cL_v} \left( X^{(0)}_{\cL_v}\right)^{+}\right)\left(\boldsymbol{\phi}_{u}\bm{Y}(z)+\boldsymbol\iota_u\right) < \bm 1 - ( X^{(0)}_{\cL_v^c})^\top\left(( X^{(0)}_{\cL_v})^\top\right)^{+}\cS_{\cL_v}  , \\
    - \frac{1}{\tilde{\lambda}n_T}( X^{(0)}_{\cL_v^c})^\top\left(I_{n_T} -  X^{(0)}_{\cL_v} \left( X^{(0)}_{\cL_v}\right)^{+}\right)\left(\boldsymbol{\phi}_{u}\bm{Y}(z)+\boldsymbol\iota_u\right) < \bm 1 + ( X^{(0)}_{\cL_v^c})^\top\left(( X^{(0)}_{\cL_v})^\top\right)^{+}\cS_{\cL_v}
       \end{aligned}
\right\}, \\
  &= \left\{
    \begin{aligned}       
      \frac{1}{\tilde{\lambda}n_T}( X^{(0)}_{\cL_v^c})^\top\left(I_{n_T} -  X^{(0)}_{\cL_v} \left( X^{(0)}_{\cL_v}\right)^{+}\right)\left(\boldsymbol{\phi}_{u}\bm{a}+\boldsymbol{\phi}_{u} \bm{b}z+\boldsymbol\iota_u\right) < \bm 1 - ( X^{(0)}_{\cL_v^c})^\top\left(( X^{(0)}_{\cL_v})^\top\right)^{+}\cS_{\cL_v}  , \\
    -     \frac{1}{\tilde{\lambda}n_T}( X^{(0)}_{\cL_v^c})^\top\left(I_{n_T} -  X^{(0)}_{\cL_v} \left( X^{(0)}_{\cL_v}\right)^{+}\right)\left(\boldsymbol{\phi}_{u}\bm{a}+\boldsymbol{\phi}_{u} \bm{b}z+\boldsymbol\iota_u\right) < \bm 1 + ( X^{(0)}_{\cL_v^c})^\top\left(( X^{(0)}_{\cL_v})^\top\right)^{+}\cS_{\cL_v} 
       \end{aligned}
\right\}, \\
    &= \left\{ \boldsymbol{\nu}_1 z \leq \boldsymbol{\kappa}_1 \right\}
\end{aligned}$%
}\\

where
\[
\boldsymbol{\nu}_1 =     \begin{pmatrix}          \frac{1}{\tilde{\lambda}n_T}( X^{(0)}_{\cL_v^c})^\top\left(I_{n_T} -  X^{(0)}_{\cL_v} \left( X^{(0)}_{\cL_v}\right)^{+}\right)\boldsymbol{\phi}_{u} \bm{b} \\
    -\frac{1}{\tilde{\lambda}n_T}( X^{(0)}_{\cL_v^c})^\top\left(I_{n_T} -  X^{(0)}_{\cL_v} \left( X^{(0)}_{\cL_v}\right)^{+}\right)\boldsymbol{\phi}_{u} \bm{b}
    \end{pmatrix},
\]
\[
\resizebox{\linewidth}{!}{$
\boldsymbol{\kappa}_1 = 
\begin{pmatrix}
    \bm 1 - \left( X^{(0)}_{\cL_v^c}\right)^\top\left(\left( X^{(0)}_{\cL_v}\right)^\top\right)^{+}\cS_{\cL_v} - \dfrac{1}{\tilde{\lambda}n_T}\left( X^{(0)}_{\cL_v^c}\right)^\top\left(I_{n_T} -  X^{(0)}_{\cL_v} \left( X^{(0)}_{\cL_v}\right)^{+}\right)(\boldsymbol{\phi}_{u} \bm{a} +  \boldsymbol\iota_u) \\
    \bm 1 + \left( X^{(0)}_{\cL_v^c}\right)^\top\left(\left( X^{(0)}_{\cL_v}\right)^\top\right)^{+}\cS_{\cL_v} + \dfrac{1}{\tilde{\lambda}n_T}\left( X^{(0)}_{\cL_v^c}\right)^\top\left(I_{n_T} -  X^{(0)}_{\cL_v} \left( X^{(0)}_{\cL_v}\right)^{+}\right)(\boldsymbol{\phi}_{u} \bm{a} +  \boldsymbol\iota_u)
\end{pmatrix}.
$}
\]

Finally, the set $\mathcal{Z}_v$ can be defined as:
\[
\mathcal{Z}_v = \{ z \in \mathbb{R} \mid \boldsymbol\nu z \leq \boldsymbol\kappa \},
\]
where $\boldsymbol\nu = (\boldsymbol\nu_0 \quad \boldsymbol\nu_1)^{\top}$, $\boldsymbol\kappa
= (\boldsymbol\kappa_0 \quad \boldsymbol\kappa_1)^{\top}$.

\subsection{Proof of Lemma \ref{lem:Zt}}\label{proof:lemZt}
The set $\mathcal{Z}_t$ in \eqref{eq:cZt} can be reformulated as follows:
\[
\mathcal{Z}_t = \left\{ z \in \mathbb{R} \,\middle|\, 
\begin{array}{l}
\left(\hat{\beta}^{(0)}_{\rm TransFusion}\right)_j(z) \neq 0, \quad \forall j \in \mathcal{M}_t, \\
\left(\hat{\beta}^{(0)}_{\rm TransFusion}\right)_j(z) = 0, \quad \forall j \notin \mathcal{M}_t, \\
\text{sign} \bigg( \left(\hat{\boldsymbol\beta}^{(0)}_{\rm TransFusion}\right)_{\mathcal{M}_t}(z) \bigg) = \cS_{\mathcal{M}_t}
\end{array} 
\right\} 
\]

First, based on Eq. \eqref{eq:delta_v}, we compute 
$\hat{\boldsymbol{\delta}}_v(z)$ as follows:
\begin{equation}\label{eq:del_u}
\begin{split}
           \hat{\boldsymbol{\delta}}_v(z) &= F_v\hat{\boldsymbol{\delta}}_{\cL_v}(z)\\           &=F_v\left(\left( X^{(0)}_{\cL_v}\right)^\top X^{(0)}_{\cL_v}\right)^{-1} \left( \left( X^{(0)}_{\cL_v}\right)^\top\left(\boldsymbol{\phi}_u\bm{Y}(z) + \boldsymbol\iota_u\right) - n_T\tilde{\lambda} \cS_{\cL_v}\right ) 
\end{split}
\end{equation}

where 
\[
F_v = \left( \bm{e}_{j_1}, \bm{e}_{j_2},\dots, \bm{e}_{j_{|\cL_v|}} \right) \in \mathbb{R}^{p \times |\cL_v| }, \quad j_k \in \cL_v.
\]\\

Next, we compute $\hat{\boldsymbol{\beta}}(z)$ using Eq. \eqref{eq:beta_transfusion} in the TransFusion algorithm, after obtaining 
$\hat{\boldsymbol w}_u(z)$ and $\hat{\boldsymbol{\delta}}_v(z)$
from \eqref{eq:w_u} and \eqref{eq:del_u}, respectively:
\begin{align*}
\hat{\boldsymbol\beta}^{(0)}_{\rm TransFusion}(z) &=  \hat{\boldsymbol w}_u(z) + \hat{\boldsymbol{\delta}}_v(z) \\
&= \frac{1}{N}BE_u \left( X^\top_{\cO_u} X_{\cO_u}\right)^{-1} \left ( X^\top_{\cO_u}\bm{Y}(z)-N\lambda_0 \mathbf{\tilde{a}}_{\cO_u} \circ \cS_{\cO_u}\right ) \\ & \quad + F_v\left(\left( X^{(0)}_{\cL_v}\right)^\top X^{(0)}_{\cL_v}\right)^{-1} \left( \left( X^{(0)}_{\cL_v}\right)^\top\left(\boldsymbol{\phi}_u\bm{Y}(z) + \boldsymbol\iota_u\right) - n_T\tilde{\lambda} \cS_{\cL_v}\right ) \\
& = \boldsymbol\xi_{uv}\bm{Y}(z) + \boldsymbol\zeta_{uv}
\end{align*}
where

\resizebox{\linewidth}{!}{$
\begin{gathered}
\boldsymbol{\xi}_{uv} = \frac{1}{N} B E_u 
\left(  X^\top_{\cO_u}  X_{\cO_u} \right)^{-1} 
 X^\top_{\cO_u} + 
F_v \left( \left(  X^{(0)}_{\cL_v} \right)^\top  X^{(0)}_{\cL_v} \right)^{-1} 
\left(  X^{(0)}_{\cL_v} \right)^\top \boldsymbol{\phi}_u, \\
\boldsymbol{\zeta}_{uv} = 
-\lambda_0 B E_u 
\left(  X^\top_{\cO_u}  X_{\cO_u} \right)^{-1} 
\left( \tilde{\bm{a}}_{\cO_u} \circ \cS_{\cO_u} \right) +
F_v \left( \left(  X^{(0)}_{\cL_v} \right)^\top  X^{(0)}_{\cL_v} \right)^{-1} 
\left( \left(  X^{(0)}_{\cL_v} \right)^\top \boldsymbol\iota_u - n_T \tilde{\lambda} \cS_{\cL_v} \right).
\end{gathered}
$}\\

Then, denoting $\mathcal{M}_t^c$ as the complement of $\mathcal{M}_t$, the set $\mathcal{Z}_t$ can be rewritten as:
\begin{equation*}
\mathcal{Z}_t = \left\{ z \in \mathbb{R} \,\middle|\, 
\begin{array}{l}
\hat{\boldsymbol{\beta}}(z) = \boldsymbol\xi_{uv}\bm{Y}(z) + \boldsymbol\zeta_{uv}\\
\text{sign} \bigg( 
\left(\hat{\boldsymbol\beta}^{(0)}_{\rm TransFusion}\right)_{\mathcal{M}_t}(z) \bigg) = \cS_{\mathcal{M}_t}, \\
\hat{\boldsymbol{\beta}}_{\mathcal{M}_{t}^c}(z) = \mathbf{0} 
\end{array} 
\right\}
\end{equation*}

The two last conditions of $\mathcal{Z}_t$ then can be rewritten as:
\begin{align*}
\left\{ \text{sign} \bigg( \left(\hat{\boldsymbol\beta}^{(0)}_{\rm TransFusion}\right)_{\mathcal{M}_t}(z) \bigg) = \cS_{\mathcal{M}_t} \right\} 
&= \left\{ \cS_{\mathcal{M}_t} \circ \left(\hat{\boldsymbol\beta}^{(0)}_{\rm TransFusion}\right)_{\mathcal{M}_t}(z) > \bm 0 \right\}, \\
&= \left\{ \cS_{\mathcal{M}_t} \circ D_t\left(\boldsymbol\xi_{uv}\bm{Y}(z) + \boldsymbol\zeta_{uv}\right) > \bm 0 \right\}, \\
&= \left\{ \cS_{\mathcal{M}_t} \circ D_t\left(\boldsymbol\xi_{uv}(\bm{a}+\bm{b}z ) + \boldsymbol\zeta_{uv}\right) > \bm 0 \right\}, \\
&= \left\{ \cS_{\mathcal{M}_t} \circ D_t \left(\boldsymbol{\zeta}_{uv}+\boldsymbol\xi_{uv}\bm{a}\right)> -\cS_{\mathcal{M}_t} \circ D_t\boldsymbol\xi_{uv}\bm{b}z \right\}, \\
&= \left\{ \boldsymbol{\omega}_0 z \leq \boldsymbol{\rho}_0 \right\}
\end{align*}
where
\[
D_t =  \begin{bmatrix} \bm{e}_{j_1}^\top \\ \bm{e}_{j_2}^\top \\ \vdots \\ \bm{e}_{j_{|\mathcal{M}_t|}}^\top \end{bmatrix} \in \mathbb{R}^{|\mathcal{M}_t| \times p} \, , \, j_k \in \mathcal{M}_t,
\]
\[
\boldsymbol{\omega}_0 = -\cS_{\mathcal{M}_t} \circ D_t\boldsymbol\xi_{uv}\bm{b},\quad 
\boldsymbol{\rho}_0 = \cS_{\mathcal{M}_t} \circ D_t\left(\boldsymbol{\zeta}_{uv}+\boldsymbol\xi_{uv}\bm{a}\right).
\]
\begin{align*}
\left\{ \hat{\boldsymbol{\beta}}_{\mathcal{M}_{t}^c}(z) = 0 \right\}
&= \left\{ 0 \leq \hat{\boldsymbol{\beta}}_{\mathcal{M}_{t}^c}(z)  \leq 0 \right\}, \\
&= \left\{ 0 \leq D_t^c\left(\boldsymbol\xi_{uv}\bm{Y}(z) + \boldsymbol\zeta_{uv}\right)  
\leq 0 \right\}, \\
&= \left\{ 0 \leq D_t^c\left(\boldsymbol\xi_{uv}(\bm{a} + \bm{b}z) + \boldsymbol\zeta_{uv}\right)  
\leq 0 \right\}, \\
&= \left\{ 
\begin{aligned}
    D_t^c\boldsymbol\xi_{uv}\bm{b}z \leq     - D_t^c(\boldsymbol\xi_{uv}\bm{a} + \boldsymbol{\zeta}_{uv})\\
    - D_t^c\boldsymbol\xi_{uv}\bm{b}z \leq     D_t^c(\boldsymbol\xi_{uv}\bm{a} + \boldsymbol{\zeta}_{uv})
\end{aligned}
\right\}, \\
&= \left\{ \boldsymbol{\omega}_1 z \leq \boldsymbol{\rho}_1 \right\}
\end{align*}

where
\[
D_t^c =  \begin{bmatrix} \bm{e}_{j_1}^\top \\ \bm{e}_{j_2}^\top \\ \vdots \\ \bm{e}_{j_{|\mathcal{M}^c_t|}}^\top \end{bmatrix} \in \mathbb{R}^{|\mathcal{M}^c_t| \times p} \, , \, j_k \in \mathcal{M}^c_t,
\]
\[
\boldsymbol{\omega}_1 = 
\begin{pmatrix}
    D_t^c\boldsymbol\xi_{uv}\bm{b}\\
    -D_t^c\boldsymbol\xi_{uv}\bm{b}
\end{pmatrix},\quad 
\boldsymbol{\rho}_1 = 
\begin{pmatrix}
     -D_t^c(\boldsymbol\xi_{uv}\bm{a} + \boldsymbol{\zeta}_{uv})\\
     D_t^c(\boldsymbol\xi_{uv}\bm{a} + \boldsymbol{\zeta}_{uv})
\end{pmatrix}.
\]\\

Finally, the set $\mathcal{Z}_t$ can be defined as:
\[
\mathcal{Z}_t = \{ z \in \mathbb{R} \mid \boldsymbol\omega z \leq \boldsymbol\rho \},
\]
where $\boldsymbol\omega = (\boldsymbol\omega_0 \quad \boldsymbol\omega_1)^{\top}$, $\boldsymbol\rho
= (\boldsymbol\rho_0 \quad \boldsymbol\rho_1)^{\top}$.
\subsection{Proof of Lemma $\ref{lem:Z-OTL}$}\label{proof:Z-OTL}
The set $\mathcal{Z}_u^{\rm otl}$ can be reformulated as follows:
\[
\mathcal{Z}_u^{\rm otl} = \left\{ z \in \mathbb{R} \,\middle|\, 
\begin{array}{l}
\hat{w}^{\cI}_{j'}(z) \neq 0, \quad \forall {j'} \in \cO^{\rm otl}_u, \\
\hat{w}^{\cI}_{j'}(z) = 0, \quad \forall {j'} \notin \cO^{\rm otl}_u, \\
\text{sign} \big( \hat{\boldsymbol w}^{\cI}_{\cO^{\rm otl}_u}(z) \big) = \cS_{\cO^{\rm otl}_u}
\end{array}
\right\}
\]
We characterize $Z_u^{\rm otl}$ through derivation from the KKT  conditions of the Lasso \eqref{eq:w_OTL} as follows:
\begin{equation}\label{eq:KKT-w2}
\begin{split}
    &\frac{1}{n_\mathcal{I}}\left( X^\mathcal{I}\right)^\top\left( X^\mathcal{I} \hat{\boldsymbol w}^{\cI}(z) - \bm{Y}^\mathcal{I}(z)\right) + \lambda_{w} \cS(z) = \bm 0, \\
    &\cS_{j'} = \text{sign}(\hat{w}^{\cI}_{j'}(z)), \quad  \text{if } \hat{w}^{\cI}_{j'}(z) \neq 0,\\
    &\cS_{j'} \in (-1,1), \quad \text{if } \hat{w}^{\cI}_{j'}(z) = 0.
\end{split}
\end{equation}

Let us define:
\begin{equation*}
     P  = \left(I_{n_\mathcal{I}}, 0_{{n_\mathcal{I}} \times n_0}  \right) \in \mathbb{R}^{n_\mathcal{I} \times \left(n_\mathcal{I}+n_T\right)}
\end{equation*}
and obtain $\bm{Y}^{\mathcal{I}}(z) =  P \bm{Y}(z)$. \\

By partitioning Equation~\eqref{eq:KKT-w2} with respect to the active set $\cO^{\rm otl}_{u}$, where ${\cO^{\rm otl}_{u}}^{c}$ denotes its complement, we reformulate the KKT conditions as:
\begin{equation}\label{eq:KKT-w}
    \begin{split}
        &\frac{1}{n_\mathcal{I}}\left ( X_{\cO^{\rm otl}_u}^\mathcal{I}\right)^\top\left( X_{\cO^{\rm otl}_u}^\mathcal{I}\hat{\boldsymbol w}^{\cI}_{\cO^{\rm otl}_u}(z) -  P \bm{Y}(z)\right) + \lambda_{w}\cS_{\cO^{\rm otl}_u}(z) = \bm 0, \\
        &\frac{1}{n_\mathcal{I}}\left ( X_{{\cO^{\rm otl}_{u}}^{c}}^\mathcal{I}\right)^\top\left( X_{\cO^{\rm otl}_u}^\mathcal{I}\hat{\boldsymbol w}^{\cI}_{\cO^{\rm otl}_u}(z) -  P \bm{Y}(z)\right) + \lambda_{w}\cS_{{\cO^{\rm otl}_{u}}^{c}}(z) = \bm 0, \\
        &\text{sign} \big( \hat{\boldsymbol w}^{\cI}_{\cO^{\rm otl}_u}(z) \big) = \cS_{\cO^{\rm otl}_u},\\ 
        &\| \cS_{{\cO^{\rm otl}_{u}}^{c}} \|_{\infty} < \bm 1.
    \end{split}
\end{equation}

Solving the first two equations in \eqref{eq:KKT-w2} for $\hat{\boldsymbol w}^{\cI}_{\cO^{\rm otl}_u}(z)$ and $\cS_{{\cO^{\rm otl}_{u}}^{c}}$ yields the equivalent conditions:
\begin{align}
    &\hat{\boldsymbol w}^{\cI}_{\cO^{\rm otl}_u}(z) = \left(\left ( X_{\cO^{\rm otl}_{u}}^\mathcal{I} \right)^\top X_{\cO^{\rm otl}_{u}}^\mathcal{I}\right)^{-1}\left(\left ( X_{\cO^{\rm otl}_{u}}^\mathcal{I} \right)^\top P \bm{Y}(z)-{n_\mathcal{I}}\lambda_{\boldsymbol{w}} \cS_{\cO^{\rm otl}_u}\right), \label{eq:wW-OTL} \\
    &\resizebox{\textwidth}{!}{$
    \cS_{{\cO^{\rm otl}_{u}}^{c}} 
    = \left( X^{\mathcal{I}}_{{\cO^{\rm otl}_{u}}^{c}}\right)^\top\left(\left( X^\mathcal{I}_{\cO^{\rm otl}_u}\right)^\top \right)^{+} \cS_{\cO^{\rm otl}_u} 
 + \frac{1}{\lambda_{\boldsymbol{w}} {n_\mathcal{I}}}\left( X^{\mathcal{I}}_{{\cO^{\rm otl}_{u}}^{c}}\right)^\top\left(I_{n_\mathcal{I}} -  X^{\mathcal{I}}_{\cO^{\rm otl}_u}\left( X^{\mathcal{I}}_{\cO^{\rm otl}_u}\right)^{+}\right) P \bm{Y}(z), $} \nonumber \\
    &\operatorname{sign} \big( \hat{\boldsymbol w}^{\cI}_{\cO^{\rm otl}_u}(z) \big) = \cS_{\cO^{\rm otl}_u}, \nonumber \\ 
    &\| \cS_{{\cO^{\rm otl}_{u}}^{c}} \|_{\infty} < \bm 1. \nonumber
\end{align}

Then, the set $\mathcal{Z}_u^{\rm otl}$ can be rewritten as:
\[
\resizebox{\linewidth}{!}{$
\mathcal{Z}_u^{\rm otl} = \left\{ z \in \mathbb{R} \,\middle|\, 
\begin{aligned}
& \hat{\boldsymbol w}^{\cI}_{\cO^{\rm otl}_u}(z) = \left(\left ( X_{\cO^{\rm otl}_{u}}^\mathcal{I} \right)^\top X_{\cO^{\rm otl}_{u}}^\mathcal{I}\right)^{-1}\left(\left ( X_{\cO^{\rm otl}_{u}}^\mathcal{I} \right)^\top P \bm{Y}(z)-{n_\mathcal{I}}\lambda_{\boldsymbol{w}} \cS_{\cO^{\rm otl}_u}\right), \\
    &\cS_{{\cO^{\rm otl}_{u}}^{c}} 
    = \left( X^{\mathcal{I}}_{{\cO^{\rm otl}_{u}}^{c}}\right)^\top\left(\left( X^\mathcal{I}_{\cO^{\rm otl}_u}\right)^\top \right)^{+} \cS_{\cO^{\rm otl}_u} 
 + \frac{1}{\lambda_{\boldsymbol{w}} {n_\mathcal{I}}}\left( X^{\mathcal{I}}_{{\cO^{\rm otl}_{u}}^{c}}\right)^\top\left(I_{n_\mathcal{I}} -  X^{\mathcal{I}}_{\cO^{\rm otl}_u}\left( X^{\mathcal{I}}_{\cO^{\rm otl}_u}\right)^{+}\right) P \bm{Y}(z) , \\
& \operatorname{sign} \big( \hat{\boldsymbol w}^{\cI}_{\cO^{\rm otl}_u}(z) \big) = \cS_{\cO^{\rm otl}_u}, \\ 
& \| \cS_{{\cO^{\rm otl}_{u}}^{c}} \|_{\infty} < \bm 1.
\end{aligned}
\right\}
$}
\]

The two last conditions of $\mathcal{Z}_u^{\rm otl}$ then can be rewritten as:
\begin{align*}
&\quad \:\left\{ \text{sign} \big( \hat{\boldsymbol w}^{\cI}_{\cO^{\rm otl}_u}(z) \big) = \cS_{\cO^{\rm otl}_u} \right\} \\
&= \left\{ \cS_{\cO^{\rm otl}_u} \circ \hat{\boldsymbol w}^{\cI}_{\cO^{\rm otl}_u}(z) > \bm 0 \right\}, \\
&= \left\{ \cS_{\cO^{\rm otl}_u} \circ \left(\left ( X_{\cO^{\rm otl}_{u}}^\mathcal{I} \right)^\top X_{\cO^{\rm otl}_{u}}^\mathcal{I}\right)^{-1}\left(\left ( X_{\cO^{\rm otl}_{u}}^\mathcal{I} \right)^\top P \bm{Y}(z)-{n_\mathcal{I}}\lambda_{\boldsymbol{w}} \cS_{\cO^{\rm otl}_u}\right) > \bm 0 \right\}, \\
&= \left\{ 
    \cS_{\cO^{\rm otl}_u} \circ \left( X^{\mathcal{I}}_{\cO^{\rm otl}_u}\right)^{+} P \bm{Y}(z) > {n_\mathcal{I}}\lambda_{\boldsymbol{w}} \cS_{\cO^{\rm otl}_u} \circ \left(\left(\left ( X_{\cO^{\rm otl}_{u}}^\mathcal{I} \right)^\top X_{\cO^{\rm otl}_{u}}^\mathcal{I}\right)^{-1} \cS_{\cO^{\rm otl}_u} \right)
\right\}, \\
&= \left\{  \cS_{\cO^{\rm otl}_u} \circ \left( X^{\mathcal{I}}_{\cO^{\rm otl}_u}\right)^{+} P (\bm{a} + \bm{b}z) > {n_\mathcal{I}}\lambda_{\boldsymbol{w}} \cS_{\cO^{\rm otl}_u} \circ \left( \left(\left ( X_{\cO^{\rm otl}_{u}}^\mathcal{I} \right)^\top X_{\cO^{\rm otl}_{u}}^\mathcal{I}\right)^{-1} \cS_{\cO^{\rm otl}_u} \right)
 \right\}, \\
&= \left\{ \boldsymbol{\psi}^{\rm otl}_0 z \leq \boldsymbol{\gamma}^{\rm otl}_0 \right\}
\end{align*}
where 
\[
\boldsymbol{\psi}^{\rm otl}_0 = -\cS_{\cO^{\rm otl}_u} \circ \left( X^{\mathcal{I}}_{\cO^{\rm otl}_u}\right)^{+}  P \bm{b}\,,
\]
\[
\boldsymbol{\gamma}^{\rm otl}_0 = \cS_{\cO^{\rm otl}_u} \circ \left( X^{\mathcal{I}}_{\cO^{\rm otl}_u}\right)^{+}  P \bm{a} - {n_\mathcal{I}}\lambda_{\boldsymbol{w}} \cS_{\cO^{\rm otl}_u} \circ \left( \left(\left ( X_{\cO^{\rm otl}_{u}}^\mathcal{I} \right)^\top X_{\cO^{\rm otl}_{u}}^\mathcal{I}\right)^{-1} \cS_{\cO^{\rm otl}_u} \right).
\]

\resizebox{1\linewidth}{!}{$
\begin{aligned}
    & \quad \:\left\{ \| \cS_{{\cO^{\rm otl}_{u}}^{c}} \|_{\infty} < \bm 1 \right\} \\
    &= \left\{ -\bm 1 < \cS_{{\cO^{\rm otl}_{u}}^{c}} < \bm 1 \right\}, \\
    &= \left\{ -\bm 1 < \left( X^{\mathcal{I}}_{{\cO^{\rm otl}_{u}}^{c}}\right)^\top\left(\left( X^\mathcal{I}_{\cO^{\rm otl}_u}\right)^\top \right)^{+} \cS_{\cO^{\rm otl}_u} + \frac{1}{\lambda_{\boldsymbol{w}} {n_\mathcal{I}}}\left( X^{\mathcal{I}}_{{\cO^{\rm otl}_{u}}^{c}}\right)^\top\left (I_{n_\mathcal{I}} -  X^{\mathcal{I}}_{\cO^{\rm otl}_u}\left( X^{\mathcal{I}}_{\cO^{\rm otl}_u}\right)^{+}\right) P \bm{Y}(z) < \bm 1 \right\}, \\
    &= \left\{
    \begin{aligned}       
     \frac{1}{\lambda_{\boldsymbol{w}} {n_\mathcal{I}}}\left( X^{\mathcal{I}}_{{\cO^{\rm otl}_{u}}^{c}}\right)^\top\left (I_{n_\mathcal{I}} -  X^{\mathcal{I}}_{\cO^{\rm otl}_u}\left( X^{\mathcal{I}}_{\cO^{\rm otl}_u}\right)^{+}\right) P \bm{Y}(z) &< \bm 1 - \left( X^{\mathcal{I}}_{{\cO^{\rm otl}_{u}}^{c}}\right)^\top\left(\left( X^\mathcal{I}_{\cO^{\rm otl}_u}\right)^\top \right)^{+} \cS_{\cO^{\rm otl}_u}  , \\
    - \frac{1}{\lambda_{\boldsymbol{w}} {n_\mathcal{I}}}\left( X^{\mathcal{I}}_{{\cO^{\rm otl}_{u}}^{c}}\right)^\top\left (I_{n_\mathcal{I}} -  X^{\mathcal{I}}_{\cO^{\rm otl}_u}\left( X^{\mathcal{I}}_{\cO^{\rm otl}_u}\right)^{+}\right) P \bm{Y}(z) &< \bm 1 + \left( X^{\mathcal{I}}_{{\cO^{\rm otl}_{u}}^{c}}\right)^\top\left(\left( X^\mathcal{I}_{\cO^{\rm otl}_u}\right)^\top \right)^{+} \cS_{\cO^{\rm otl}_u} 
       \end{aligned}
\right\}, \\
  &= \left\{
    \begin{aligned}       
     \frac{1}{\lambda_{\boldsymbol{w}} {n_\mathcal{I}}}\left( X^{\mathcal{I}}_{{\cO^{\rm otl}_{u}}^{c}}\right)^\top\left (I_{n_\mathcal{I}} -  X^{\mathcal{I}}_{\cO^{\rm otl}_u}\left( X^{\mathcal{I}}_{\cO^{\rm otl}_u}\right)^{+}\right) P (\bm{a} + \bm{b}z) &< \bm 1 - \left( X^{\mathcal{I}}_{{\cO^{\rm otl}_{u}}^{c}}\right)^\top\left(\left( X^\mathcal{I}_{\cO^{\rm otl}_u}\right)^\top \right)^{+} \cS_{\cO^{\rm otl}_u}  , \\
    - \frac{1}{\lambda_{\boldsymbol{w}} {n_\mathcal{I}}}\left( X^{\mathcal{I}}_{{\cO^{\rm otl}_{u}}^{c}}\right)^\top\left (I_{n_\mathcal{I}} -  X^{\mathcal{I}}_{\cO^{\rm otl}_u}\left( X^{\mathcal{I}}_{\cO^{\rm otl}_u}\right)^{+}\right) P (\bm{a} + \bm{b}z) &< \bm 1 + \left( X^{\mathcal{I}}_{{\cO^{\rm otl}_{u}}^{c}}\right)^\top\left(\left( X^\mathcal{I}_{\cO^{\rm otl}_u}\right)^\top \right)^{+} \cS_{\cO^{\rm otl}_u} 
       \end{aligned}
\right\}, \\
    &= \left\{ \boldsymbol{\psi}^{\rm otl}_1 z \leq \boldsymbol{\gamma}^{\rm otl}_1 \right\}
\end{aligned}
$}\\

where
\[
\boldsymbol{\psi}^{\rm otl}_1 =  \begin{pmatrix}    \frac{1}{\lambda_{\boldsymbol{w}} {n_\mathcal{I}}}\left( X^{\mathcal{I}}_{{\cO^{\rm otl}_{u}}^{c}}\right)^\top\left (I_{n_\mathcal{I}} -  X^{\mathcal{I}}_{\cO^{\rm otl}_u}\left( X^{\mathcal{I}}_{\cO^{\rm otl}_u}\right)^{+}\right) P \bm{b} \\
    -\frac{1}{\lambda_{\boldsymbol{w}} {n_\mathcal{I}}}\left( X^{\mathcal{I}}_{{\cO^{\rm otl}_{u}}^{c}}\right)^\top\left (I_{n_\mathcal{I}} -  X^{\mathcal{I}}_{\cO^{\rm otl}_u}\left( X^{\mathcal{I}}_{\cO^{\rm otl}_u}\right)^{+}\right) P \bm{b}
    \end{pmatrix},
\]
\[
\resizebox{1\linewidth}{!}{$
\boldsymbol{\gamma}^{\rm otl}_1 = \begin{pmatrix}
\bm 1 - \left( X^{\mathcal{I}}_{{\cO^{\rm otl}_{u}}^{c}}\right)^\top\left(\left( X^\mathcal{I}_{\cO^{\rm otl}_u}\right)^\top \right)^{+} \cS_{\cO^{\rm otl}_u} - \frac{1}{\lambda_{\boldsymbol{w}} {n_\mathcal{I}}}\left( X^{\mathcal{I}}_{{\cO^{\rm otl}_{u}}^{c}}\right)^\top\left (I_{n_\mathcal{I}} -  X^{\mathcal{I}}_{\cO^{\rm otl}_u}\left( X^{\mathcal{I}}_{\cO^{\rm otl}_u}\right)^{+}\right) P \bm{a} \\
\bm 1 + \left( X^{\mathcal{I}}_{{\cO^{\rm otl}_{u}}^{c}}\right)^\top\left(\left( X^\mathcal{I}_{\cO^{\rm otl}_u}\right)^\top \right)^{+} \cS_{\cO^{\rm otl}_u} + \frac{1}{\lambda_{\boldsymbol{w}} {n_\mathcal{I}}}\left( X^{\mathcal{I}}_{{\cO^{\rm otl}_{u}}^{c}}\right)^\top\left (I_{n_\mathcal{I}} -  X^{\mathcal{I}}_{\cO^{\rm otl}_u}\left( X^{\mathcal{I}}_{\cO^{\rm otl}_u}\right)^{+}\right) P \bm{a}
\end{pmatrix}.
$}
\]\\

Finally, the set $\mathcal{Z}_u^{\rm otl}$ can be defined as:
\[
\mathcal{Z}_u^{\rm otl} = \{ z \in \mathbb{R} \mid \boldsymbol{\psi}^{\rm otl} z \leq \boldsymbol{\gamma}^{\rm otl} \},
\]
where $\boldsymbol{\psi}^{\rm otl} = (\boldsymbol{\psi}^{\rm otl}_0 \quad \boldsymbol{\psi}^{\rm otl}_1)^{\top}$, $\boldsymbol{\gamma}^{\rm otl} = (\boldsymbol{\gamma}^{\rm otl}_0 \quad \boldsymbol{\gamma}^{\rm otl}_1)^{\top}$.\\

Thus, the set $\mathcal{Z}_u^{\rm otl}$ can be identified by solving a system of linear inequalities with respect to $z$.\\

Next, we compute $\hat{\boldsymbol w}^{\cI}_u(z)$ based on $\hat{\boldsymbol w}^{\cI}_{\cO^{\rm otl}_u}(z)$ in $\eqref{eq:wW-OTL}$:
\begin{equation}\label{eq:w_u-OTL}
\begin{split}
        \hat{\boldsymbol w}^{\cI}_u(z) &=  E^{\rm otl}_u\hat{\boldsymbol w}^{\cI}_{\cO^{\rm otl}_u}(z)\\ &= E^{\rm otl}_u \left(\left ( X_{\cO^{\rm otl}_{u}}^\mathcal{I} \right)^\top X_{\cO^{\rm otl}_{u}}^\mathcal{I}\right)^{-1}\left(\left ( X_{\cO^{\rm otl}_{u}}^\mathcal{I} \right)^\top P \bm{Y}(z)-{n_\mathcal{I}}\lambda_{\boldsymbol{w}} \cS_{\cO^{\rm otl}_u}\right)
\end{split}
\end{equation}
where $E^{\rm otl}_u = \left( \bm{e}_{j_1}, \bm{e}_{j_2},\dots, \bm{e}_{j_{|\cO^{\rm otl}_u|}} \right) \in \mathbb{R}^{p \times |\cO^{\rm otl}_u| }, \quad j_k \in \cO^{\rm otl}_u$.\\

The proofs for $\mathcal{Z}^{\rm otl}_v$ and $\mathcal{Z}^{\rm otl}_t$  proceed in a manner analogous to those for $\mathcal{Z}_v$ $\mathcal{Z}_t$ in Appendices \ref{proof:lemZv} and \ref{proof:lemZt}, with the vectors $\boldsymbol{\nu}^{\rm otl}, \boldsymbol{\kappa}^{\rm otl},  \boldsymbol{\omega}^{\rm otl},  \boldsymbol{\rho}^{\rm otl}$ taking forms corresponding to $\boldsymbol{\nu}, \boldsymbol{\kappa},  \boldsymbol{\omega},  \boldsymbol{\rho}$, respectively, where
\[ Q^{\rm otl} = \left(0_{n_T \times n_{\mathcal{I}}}, I_{n_T} \right) \in \mathbb{R}^{n_T \times \left(n_\mathcal{I}+n_0\right)},\]
\[
F^{\rm otl}_v = \left( \bm{e}_{j_1}, \bm{e}_{j_2},\dots, \bm{e}_{j_{|\cL^{\rm otl}_v|}} \right) \in \mathbb{R}^{p \times |\cL^{\rm otl}_v| }, \quad j_k \in \cL^{\rm otl}_v,
\]
\[
\boldsymbol{\phi}^{\rm otl}_u = Q^{\rm otl} -  X^{(0)}E^{\rm otl}_u \left(\left ( X_{\cO^{\rm otl}_{u}}^\mathcal{I} \right)^\top X_{\cO^{\rm otl}_{u}}^\mathcal{I}\right)^{-1}\left ( X_{\cO^{\rm otl}_{u}}^\mathcal{I} \right)^\top P ,
\]
\[
\boldsymbol{\iota}^{\rm otl}_u = {n_\mathcal{I}}\lambda_{\boldsymbol{w}}  X^{(0)}E^{\rm otl}_u \left(\left ( X_{\cO^{\rm otl}_{u}}^\mathcal{I} \right)^\top X_{\cO^{\rm otl}_{u}}^\mathcal{I}\right)^{-1}\cS_{\cO^{\rm otl}_u},
\]
\resizebox{\linewidth}{!}{$
\begin{gathered}
\boldsymbol{\xi}^{\rm otl}_{uv} = E^{\rm otl}_u \left(\left ( X_{\cO^{\rm otl}_{u}}^\mathcal{I} \right)^\top X_{\cO^{\rm otl}_{u}}^\mathcal{I}\right)^{-1}\left ( X_{\cO^{\rm otl}_{u}}^\mathcal{I} \right)^\top P  + 
F^{\rm otl}_v \left( \left(  X^{(0)}_{\cL^{\rm otl}_v} \right)^\top  X^{(0)}_{\cL^{\rm otl}_v} \right)^{-1} 
\left(  X^{(0)}_{\cL^{\rm otl}_v} \right)^\top \boldsymbol{\phi}^{\rm otl}_u, \\
\boldsymbol{\zeta}^{\rm otl}_{uv} = 
 - {n_\mathcal{I}}\lambda_{\boldsymbol{w}}E^{\rm otl}_u \left(\left ( X_{\cO^{\rm otl}_{u}}^\mathcal{I} \right)^\top X_{\cO^{\rm otl}_{u}}^\mathcal{I}\right)^{-1} \cS_{\cO^{\rm otl}_u} + F^{\rm otl}_v \left( \left(  X^{(0)}_{\cL^{\rm otl}_v} \right)^\top  X^{(0)}_{\cL^{\rm otl}_v} \right)^{-1} 
\left( \left(  X^{(0)}_{\cL^{\rm otl}_v} \right)^\top \boldsymbol{\rho}^{\rm otl}_u - n_T \lambda_{\boldsymbol{\delta}} \cS_{\cL^{\rm otl}_v} \right).
\end{gathered}%
$}
 
\end{appendices}

\end{document}